\documentclass[journal]{IEEEtran}
\usepackage{amsmath,amssymb,amsfonts}
\usepackage{bm}
\usepackage{cite}
\usepackage{algorithm}
\usepackage{algorithmic}
\usepackage{graphicx}
\usepackage{caption}
\usepackage{lipsum}
\usepackage{enumerate}
\usepackage{amsthm}
\usepackage{amsmath}
\usepackage{amssymb}
\usepackage{color}
\usepackage{booktabs}
\newcommand{\Exp}{\mathop{\mathbb E}\displaylimits}
\newtheorem{definition}{Definition}
\newtheorem{lemma}{Lemma}
\newtheorem{theorem}{Theorem}

\newtheorem{remark}{Remark}

\usepackage{bm}
\usepackage{threeparttable}

\newcommand\VECTOR{} 
\newcommand\SPACE{\mathcal}  

\ifCLASSOPTIONcompsoc
  \usepackage[caption=false,font=normalsize,labelfont=sf,textfont=sf]{subfig}
\else
  \usepackage[caption=false,font=footnotesize]{subfig}
\fi
\ifCLASSINFOpdf
\else
\fi
\begin{document}

%
\title{Fixed-Dimensional and Permutation Invariant State Representation of Autonomous Driving}
%
%
%

\author{Jingliang Duan, Dongjie Yu, Shengbo Eben Li*, Wenxuan Wang,  Yangang Ren, Ziyu Lin and~Bo Cheng
\thanks{This work is supported by Beijing Science and Technology Plan Project with Z191100007419008. It is also partially supported by Tsinghua-Toyota Joint Research Institute Cross-discipline Program. Jingliang Duan and Dongjie Yu contributed equally to this work. All correspondences should be sent to S. Li with email: lisb04@gmail.com. }
\thanks{J. Duan, D. Yu, S. Li, W. Wang, Y. Ren, Z. Lin, and B. Cheng are with State Key Lab of Automotive Safety and Energy, School of Vehicle and Mobility, Tsinghua University, Beijing, 100084, China. They are also with Center for Intelligent Connected Vehicles and Transportation, Tsinghua University. {\tt\small Email: duanjl15@163.com; (ydj20, wang-wx18, ryg18, linzy17)@mails.tsinghua.edu.cn; (lishbo, chengbo)@tsinghua.edu.cn}.
}
\thanks{J. Duan is also with the Department of Electrical and Computer Engineering, National University of Singapore, Singapore. {\tt\small Email: duanjl@nus.edu.sg}.
}
}

 \maketitle

\begin{abstract}
In this paper, we propose a new state representation method, called encoding sum and concatenation (ESC), to describe the environment observation for decision-making in autonomous driving. Unlike existing state representation methods, ESC is applicable to the situation where the number of surrounding vehicles is variable and eliminates the need for manually pre-designed sorting rules, leading to higher representation ability and generality. The proposed ESC method introduces a feature neural network (NN) to encode the real-valued feature of each surrounding vehicle into an encoding vector, and then adds these vectors up to obtain the representation vector of the set of surrounding vehicles. Then, a fixed-dimensional and permutation-invariance state representation can be obtained by concatenating the set representation with other variables, such as indicators of the ego vehicle and road. By introducing the sum-of-power mapping, this paper has further proved that the injectivity of the ESC state representation can be guaranteed if the output dimension of the feature NN is greater than the number of variables of all surrounding vehicles. This means that the ESC representation can be used to describe the environment and taken as the inputs of learning-based policy functions. Experiments demonstrate that compared with the fixed-permutation representation method, the policy learning accuracy based on ESC representation is improved by 62.2\%.
\end{abstract}


%
\IEEEpeerreviewmaketitle

\section{Introduction}
%
%
%
%
\IEEEPARstart{A}{utonomous} driving has become a research hotspot since it can enhance road safety, ease road congestion, free human drivers, etc. Decision-making is the core component of achieving high-level autonomous driving. Although rule-based methods have been widely used to realize decision-making, manually encoding rules is not always feasible due to the highly dynamic and stochastic nature of driving scenarios \cite{katrakazas2015rule-based,montemerlo2008junior}. The learning-based method is a promising technology to realize high-level autonomous driving by directly learning a parameterized policy that maps state representations to actions from data using supervised learning or reinforcement learning (RL) \cite{sutton2018reinforcement}.  Recent learning-based decision-making researches tend to use multi-layer neural networks (NNs) to represent the policy due to their remarkable fitting and generalization capabilities \cite{lecun2015deep,duan2021distributional,duan2021cadp}. According to the 
state representation methods, the learning-based decision making can be divided into two categories: (1) end-to-end (E2E) decision making, which directly maps the raw sensors outputs to driving decisions, and (2) tensor-to-end (T2E) decision making, which describes states using real-valued representations, such as velocity and position. 

The E2E decision-making method has been widely investigated during the last two decades, because it reduces the need for perception algorithms. In the late 1980s, Pomerleau built the first end-to-end autonomous driving system, called ALVINN, that took images consisting of ${32 \times 32}$ binary values and an ${8 \times 32}$ matrix from a laser range finder as inputs and output steering angles \cite{pomerleau1989alvinn}. After training based on 1200 labeled samples, the NAVLAB vehicle equipped with ALVINN could drive in a 400m road without obstacles at the speed of 1m/s. Similarly, NVIDIA trained a convolutional driving policy NN for autonomous highway driving, which describes states using
images from a single front-facing camera paired with the steering angles  \cite{bojarski2016NVIDIA_e2e_1,bojarski2017NVIDIA_e2e_2}. In addition to supervised learning methods, Lillicrap \emph{et al.} (2016) employed an RL algorithm, called DDPG, to learn a policy NN for lane-keeping on the TORCS simulation platform, which took simulated images as inputs and output acceleration quantity and steering wheel angles \cite{lillicrap2015DDPG}. Besides, many other related works on E2E decision-marking for autonomous driving can be found in \cite{jaritz2018end,lecun2004dave,chen2015deepdriving,wymann2000torcs,kendall2019DDPGdriving,perot2017end,wolf2017learning, liang2018cirl}. Since there is a great difference between the sensor outputs of the simulated environment and the actual vehicle, the learned policy based on simulated perception is difficult to apply to real vehicles, or only applicable to simple driving tasks such as lane-keeping \cite{pomerleau1989alvinn,kendall2019DDPGdriving}. Besides, 
the sensor outputs are also sensitive to the configuration of vehicle sensors, which limits the generalization of E2E decision-making methods in different vehicles.

Compared with E2E decision-making that takes raw sensors information as states, preliminary studies showed that real-valued representations
perform better, due to the reduced state space being easier to learn and the real values making it easier for the system to generalize \cite{isele2018navigating}. Besides, the driving style and intention represented by pre-designed values can also be taken as policy inputs to further improve driving performance \cite{li2020behavior}. Therefore, T2E decision-making has achieved great success in autonomous driving \cite{mirchevska2018high,wang2017formulation,wang2018reinforcement, wang2019continuous,ma2020alternating}. Duan \emph{et al.} (2020) represented driving states using a 26-dimensional vector, consisting of indicators of the ego vehicle, the road and the nearest four vehicles, realizing smooth and safe decision making on a simulated 2-lane highway via RL \cite{duan2020hierarchical}. Guan \emph{et al.} (2020) included a total of 16 variables from the ego vehicle and seven surrounding vehicles (position, speed, etc.) in the state representation to handle the cooperative longitudinal decision-making in a virtual intersection \cite{guan2020centralized}. The information of different vehicles is sorted according to a pre-designed order to form the final state vector.  

In summary, the T2E method needs to concatenate perception information of the ego vehicle, surrounding vehicles and roads into a state vector and then perform policy learning based on the vectorized state space. Although T2E has shown its advantages in terms of policy performance and generalization ability to vehicles with different sensor systems, it suffers from two challenges: (1) dimension sensitive problem and (2) permutation sensitive problem. The former means that T2E can only consider a fixed number of surrounding vehicles since the input dimension of the parameterized policy must be a predetermined value \cite{isele2018navigating, duan2020hierarchical, guan2020centralized}. The latter indicates that the information of surrounding vehicles needs to be 
permuted according to manually designed sorting rules because different permutations lead to different state representations and policy outputs \cite{mirchevska2018high, wang2017formulation, wang2018reinforcement, wang2019continuous}. It is usually difficult to design a proper sorting order for complex driving scenarios such as intersections. These two challenges will not only limit the generality of T2E for different driving scenarios, but also hurt the performance of the learned policy.

In this paper, we propose a new state representation method, called encoding sum and concatenation (ESC), to describe the environment observation for learning-based decision making in autonomous driving. The main contributions and advantages of this paper are as follows:
\begin{enumerate}
    \item The proposed ESC method introduces a feature NN to encode the real-valued feature of each surrounding vehicle into an encoding vector, and then adds these vectors up to obtain the representation vector of the set of surrounding vehicles. A fixed-dimensional and permutation-invariance state representation is obtained by concatenating the set representation with other variables, such as indicators of the ego vehicle and road. Different from the fixed-permutation representation method used in existing T2E studies \cite{mirchevska2018high,wang2017formulation,wang2018reinforcement, wang2019continuous,duan2020hierarchical,guan2020centralized}, ESC  is  applicable  to  the  situation  where  the  number  of surrounding vehicles is variable and eliminates the need for manually pre-designed sorting rules, leading to higher representation ability and generality.
    \item  By introducing the sum-of-power mapping, we have further proved that the injectivity of the ESC state representation can be guaranteed if the output dimension of  the feature NN is greater than  the number of variables of all surrounding vehicles. This means that the ESC representation can be used to injectively describe the environment. Besides, we further show that, by taking the ESC representation as policy inputs, we can find the nearly optimal feature NN and policy NN by simultaneously optimizing them using gradient-based updating.
    \item Function approximation experiments on six policy learning benchmarks demonstrate that, compared with the fixed-permutation representation method used in \cite{mirchevska2018high,wang2017formulation,wang2018reinforcement, wang2019continuous,duan2020hierarchical,guan2020centralized}, the  policy learning error based on the ESC representation is reduced by 62.2\%.
\end{enumerate}

In Section \ref{sec:formulation}, we describe the state representation problem, and analyze the effect of  dimension sensitivity and permutation sensitivity on policy learning. Section \ref{sec:ESC} proposes the ESC state representation method. In Section \ref{sec:exp}, we present experimental results that show the efficacy of ESC. Section \ref{sec:conclusion} concludes this paper.

{\bf{Notation:}} $\mathbb{R}^d$ denotes the set of $d$-dimensional real-valued vectors. $\mathbb{N}$ denotes the set of natural numbers. $\mathcal{O}$ denotes the set of all observed information. $\mathcal{X}$ denotes the set of surrounding vehicles. $x$ denotes the real-valued feature vector of each surrounding vehicle. $s$ denotes the state vector. $M$ denotes the number of surrounding vehicles within the perception range.

\section{Problem Description}
\label{sec:formulation}

In this section, we first describe the state representation problem. Then, we analyze the effect of dimension sensitive and permutation sensitive issues on the performance, generality, and sample complexity of policy learning.

\subsection{Observation and State}
We denote the observation set of driving scenarios as $\mathcal{O}\in \overline{\mathcal{O}}$, which consists of: (a) the information set of surrounding vehicles $\mathcal{X}=\{x_1,x_2,\cdots,x_M\}$, where $x_{i} \in \mathbb{R}^{d_1}$ is the real-valued feature vector of the $i$th surrounding vehicle, and (b) the feature vector containing other information related to the driving task $x_{\rm else} \in \mathbb{R}^{d_2}$, such as indicators of the ego vehicle and road geometry. Thus, $\mathcal{O}=\{\mathcal{X},x_{\rm else}\}$. The set size $M$ of $\mathcal{X}$, i.e., the number of surrounding vehicles within the perception range of the ego car, is constantly changing due to the dynamic nature of the traffic. Assuming that the range of the number of surrounding vehicles is $[1,N]\cap\mathbb{N}$, the space of $\mathcal{X}$ can be denoted as $\overline{\mathcal{X}}=\{\mathcal{X}|\mathcal{X}=\{x_1,\cdots,x_{M}\},x_i\in\mathbb{R}^{d_1},i\le M,M\in[1,N]\cap\mathbb{N}\}$, i.e., $\mathcal{X}\in\overline{\mathcal{X}}$. Noted that the subscript $i$ of $x_{i}$ in $\mathcal{X}$ represents the ID of a certain surrounding vehicle. For example,  $\mathcal{X}=\{x_M,x_{M-1},\cdots,x_1\}$ indicates that all vehicles are arranged in descending order according to the ID of each surrounding vehicle. Different permutations of these vehicles do not have an essential distinction since they represent the same traffic situation. But when they are input into some policy functions, the order matters.

We denote the mapping from the observation set $\SPACE{O}$ to state representation $\VECTOR{s}$ as $U(\SPACE{O})$, i.e.,
\begin{equation}
\label{eq.mapping_s}
s=U(\SPACE{O})=U(\SPACE{X},x_{\rm else}).
\end{equation}
Current T2E researches usually concatenate the variables in $\mathcal{O}$ to obtain the state representation vector $s$. According to the permutation of surrounding vehicles $x_i$ in $s$, there are two commonly used approaches: (1) all-permutation (AP) representation and (2) fixed-permutation (FP) representation. The AP method aims to consider all possible permutations of surrounding vehicles in $s$,
\begin{equation}
\label{eq.all_mapping}
\VECTOR{s} = U_{\rm AP}(\SPACE{O})= [x_{\varsigma(1)}^\top,\cdots, x_{\varsigma(M)}^\top,x_{\rm else}^\top]^\top,
\end{equation}
where $U_{\rm AP}(\SPACE{O}): \overline{\mathcal{O}}\rightarrow \mathbb{R}^{Md_1+d_2}$ denotes the AP mapping and $\varsigma$ represents any possible permutation. For example, this representation method will take $s_1=[x_1,\cdots,x_M]$ and $s_2 = [x_M,\cdots,x_1]$ as two different states although they represents the same traffic situation.
Unlike the AP method, the FP method only considers one permutation, which arranges the objects in $\mathcal{X}$ according to a pre-designed sorting rule $o$, i.e.,
\begin{equation}
\label{eq.order_mapping}
s= U_{\rm FP}(\SPACE{O})=[x_{o(1)}^{\top},\cdots,x_{o(M)}^{\top},x_{\rm else}^{\top}]^{\top},
\end{equation}
where $U_{\rm FP}(\SPACE{O}): \overline{\mathcal{O}}\rightarrow \mathbb{R}^{Md_1+d_2}$ denotes the FP mapping.

According to \eqref{eq.all_mapping} and \eqref{eq.order_mapping}, the change of vehicle number $M$ or the permutation of surrounding vehicles may lead to different state vectors $s$, bringing two challenges: (1) dimension sensitivity and (2) permutation sensitivity.  To find a better state representation method, it is necessary first to analyze the impact of these two issues on policy learning.

\subsection{Dimension Sensitivity}
The state dimension of AP and FP methods is ${\text{dim}}(U_{\rm AP}(\SPACE{O}))={\text{dim}}(U_{\rm FP}(\SPACE{O}))=Md_1+d_2$, which is proportional to the number of surrounding vehicles $M$. Since $M\in[1,N]\cap\mathbb{N}$ is constantly changing during driving, ${\text{dim}}(s)$ is not a fixed value. However, the input dimension of the parameterized policy must be a predetermined fixed value due to the structure of the approximate functions, such as neural network (NN) and polynomial functions. This means that T2E methods based on AP or FP representation are only valid when the number of surrounding vehicles is fixed \cite{mirchevska2018high, wang2017formulation, wang2018reinforcement, wang2019continuous}. Assuming that only $Z$ surrounding vehicles are considered, as shown in Figure \ref{fig.vehnumber_problem},  when $M > Z$, we need to select $Z$ vehicles from $\mathcal{X}$ based on pre-designed rules. When $M<Z$, we need to add $Z-M$ virtual vehicles far away from the ego to meet the input requirement of the policy function without affecting decision-making. The former will 
lead to information loss, while the latter will bring information redundancy. Therefore, it is crucial to select an appropriate value of $Z$ according to the requirements of different driving tasks, which also limits the generality of AP and FP methods.
\begin{figure}[!htb]
\centering
\captionsetup{singlelinecheck = false,labelsep=period, font=small}
\captionsetup[subfigure]{justification=centering}
        \subfloat[]{\label{fig:less_vehicle}\includegraphics[width=0.162\textwidth]{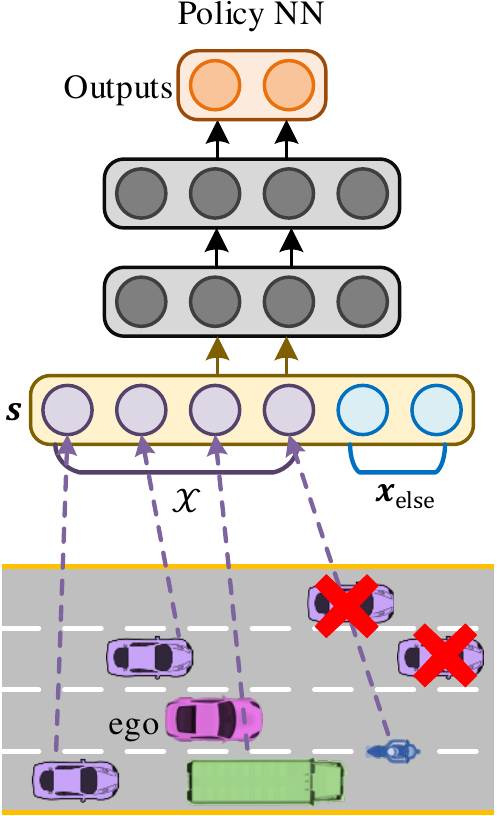}}
        \hspace{.06in}
        \subfloat[]{\label{fig:more_vehicle}\includegraphics[width=0.2\textwidth]{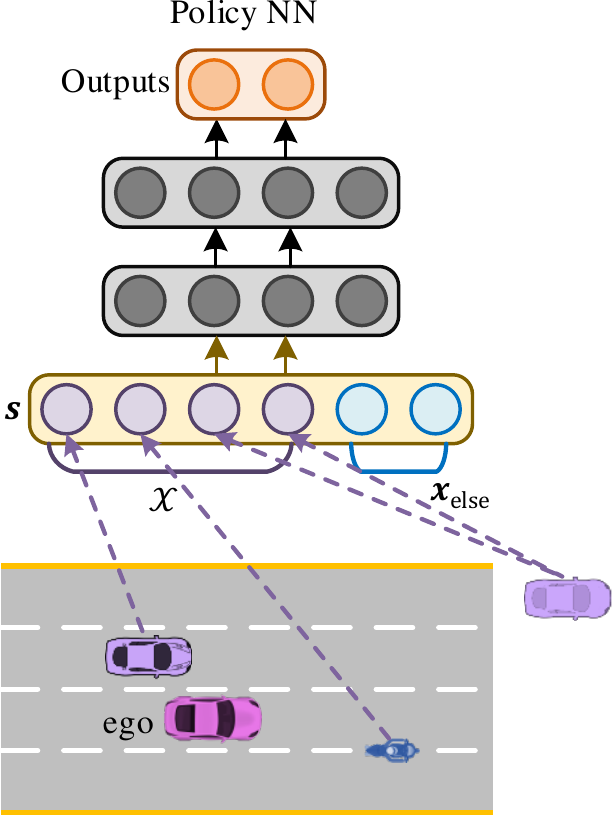}} 
        \caption{Dimension sensitivity. (a) $M>Z$. (b) $M<Z$. When $M>Z$, some surrounding vehicles cannot be input into the policy function due to dimensional limitations; while when  $M<Z$, we need to supplement the policy inputs with the information of virtual vehicles.}
    \label{fig.vehnumber_problem}
\end{figure}

\subsection{Permutation Sensitivity}

As illustrated in Figure \ref{fig.vehpermutation_problem}, assuming the number of surrounding vehicles $M$ is fixed, different permutations of $x_i$ correspond to different state vector $s$, thereby leading to different policy outputs. In other words, $s$ and policy outputs are permutation sensitive to the order of surrounding vehicles. However, a reasonable driving decision should be permutation invariant to the order of objects in $\mathcal{X}$ because all possible permutations correspond to the same driving scenario. To analyze the effect of permutation sensitivity, we first define the permutation invariant function as follows.
\begin{figure}[!htb]
\captionsetup{singlelinecheck = false,labelsep=period, font=small}
    \centering{\includegraphics[width=0.35\textwidth]{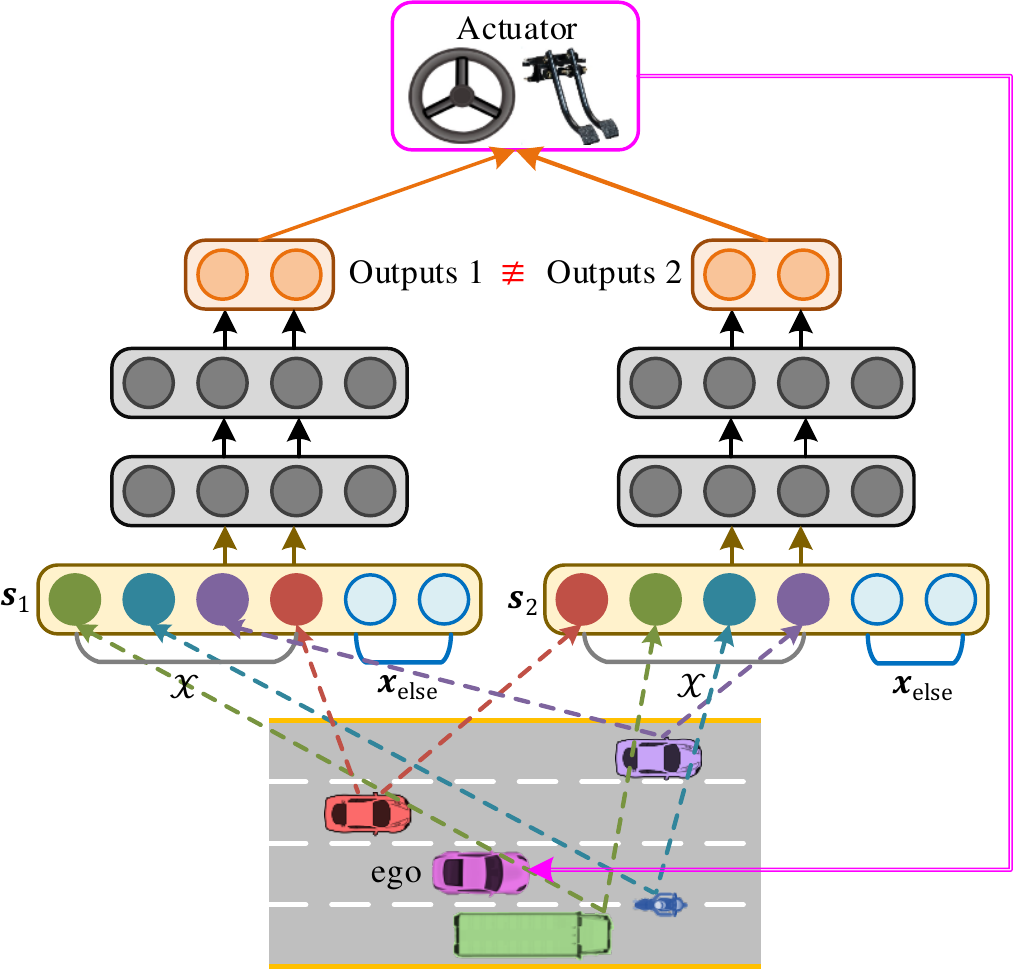}}
    \caption{Permutation sensitivity. For the same driving scenario, different permutations bring different policy inputs, which may result in different policy outputs.}
    \label{fig.vehpermutation_problem}
\end{figure}

\begin{definition}\label{defi.pi_fcn}
(Permutation Invariant Function). Function $F:\overline{\mathcal{X}} \times \mathbb{R}^{d_2}\rightarrow\mathcal{Y}$ is permutation invariant to the order of objects in the set $\mathcal{X}$ if $F(\{x_1,\cdots,x_M\}, x_{\rm else})\equiv F(\{x_{\varsigma(1)},\cdots,x_{\varsigma(M)}\}, x_{\rm else})$ for any permutation $\varsigma$.
\end{definition}
For example, $F(\{x_1,\cdots,x_M\},x_{\rm else})=\big\|\sum_{x\in \SPACE{X}}x\big\|_2+\big\|x_{\rm else}\big\|_2$ is a permutation invariant function w.r.t. $\SPACE{X}$. Similarly, we define the permutation sensitive function as
\begin{definition}\label{defi.ps_fcn}
(Permutation Sensitive Function). Function $F:\overline{\mathcal{X}} \times \mathbb{R}^{d_2}\rightarrow\mathcal{Y}$ is permutation sensitive to the order of objects in the set $\mathcal{X}$ if $\exists \varsigma$ such that $F(\{x_1,\cdots,x_M\}, x_{\rm else})\not\equiv F(\{x_{\varsigma(1)},\cdots,x_{\varsigma(M)}\}, x_{\rm else})$.
\end{definition}

We denote the expected driving policy as $F_{\rm PI}(\SPACE{X}, x_{\rm else}): \overline{\mathcal{X}} \times \mathbb{R}^{d_2}\rightarrow\mathcal{Y}$, which is permutation invariant w.r.t. $\mathcal{X}$. The objective of T2E decision-making methods is to learn a parameterized policy $\pi$, which takes $U$ as inputs, such that
\begin{equation}
\label{eq.PS_appro_PI}
\begin{aligned}
\pi(
U(\{x_{\varsigma(1)},\cdots, x_{\varsigma(M)}&\},x_{\rm else});\psi^*)\approx F_{\rm PI}(\SPACE{X},x_{\rm else}), \\
& \forall \varsigma,\forall\SPACE{X}\in \overline{\SPACE{X}},\forall x_{\rm else}\in\mathbb{R}^{d_2},
\end{aligned}
\end{equation}
where $\psi$ is the policy parameters and $^*$ indicates that the parameters are optimal. An effective mapping $U$ will significantly reduce the sample complexity and error of policy learning.

For the AP representation method in \eqref{eq.all_mapping}, the policy is learned by minimizing the following loss
\begin{equation}
\label{eq.all_permutaion}
\min_{\VECTOR{\psi}} \Exp_{\substack{\SPACE{X}\in\overline{\SPACE{X}},x_{\rm else}\in\mathbb{R}^{d_2},\forall \varsigma}}\big(\pi(U_{\rm AP}(\SPACE{O});\VECTOR{\psi})-F_{\rm PI}(\SPACE{X},x_{\rm else})\big)^2.
\end{equation}
The challenge faced by this method is that there are $M!$ permutations for a particular set $\mathcal{X}$ containing $M$ surrounding vehicles. This indicates that one driving scenario will correspond to $M!$ different state representations, which greatly increases the sample complexity.

For the FA representation method in \eqref{eq.order_mapping}, the policy can be found by minimizing
\begin{equation}
\label{eq.order_permutation}
\min_{\VECTOR{\psi}} \Exp_{\substack{\SPACE{X}\in\overline{\SPACE{X}},x_{\rm else}\in\mathbb{R}^{d_2}}}\big(\pi(U_{\rm FP}(\SPACE{O});\VECTOR{\psi})-F_{\rm PI}(\SPACE{X},x_{\rm else})\big)^2.
\end{equation}
The pre-designed order $o$ of FA guarantees the permutation invariance of the policy $\pi(U_{\rm FP}(\SPACE{O});\VECTOR{\psi})$ w.r.t. $\mathcal{X}$, reducing the sample complexity compared with AP methods. However, it may break the continuity of the policy function w.r.t. each element in $\mathcal{X}$, i.e.,
\begin{equation}
\begin{aligned}
\lim_{\VECTOR{x}_1'\rightarrow\VECTOR{x}_1,\cdots,\VECTOR{x}_M'\rightarrow\VECTOR{x}_M}&\pi(\VECTOR{x}_{{o}(1)}',\cdots, \VECTOR{x}_{{o}(M)}',\VECTOR{x}_{\rm else};\VECTOR{\psi})\\
&\not\equiv \pi(\VECTOR{x}_{{o}(1)},\cdots, \VECTOR{x}_{{o}(M)},\VECTOR{x}_{\rm else};\VECTOR{\psi}),\ \forall\VECTOR{\psi}.
\end{aligned}
\end{equation}
Since the position of each surrounding vehicle is dynamically changing during driving, the position of $x_i$ in $U_{\rm FP}(\SPACE{O})$ may change at a certain time, resulting in a sudden change in the state $s$ and policy output $\pi(U_{\rm FP}(\SPACE{O});\VECTOR{\psi})$. For example, the rear vehicle at the current moment may become the preceding vehicle at a certain moment in the future by overtaking the ego vehicle. In particular, we will give a special case below for further explanation. Let $x_{\rm else}\in\emptyset$, and $\SPACE{X}=\{[j,2]^\top,[1,5]^\top\}$, where $j$ is a variable. The rule $o$ sorts $\mathcal{X}$ in increasing order according to the first element of $x_i$. It follows that when $j\le 1$, $\VECTOR{x}_{{o}(1)}=[j,2]^\top$ and $\VECTOR{x}_{{o}(2)}=[1,5]^\top$; when $j> 1$, $\VECTOR{x}_{{o}(1)}=[1,5]^\top$ and $\VECTOR{x}_{{o}(2)}=[j,2]^\top$. It can be seen that the permutation of objects in $\mathcal{X}$ has changed around $j=1$, which may cause a sudden change in policy outputs, i.e.,
\begin{equation}
\begin{aligned}
\lim_{j\rightarrow 1^-}&\pi(U_{\rm FP}(\SPACE{O});\VECTOR{\psi})=\pi([1,2,1,5]^\top;\VECTOR{\psi})\\
&\not\equiv
\lim_{j\rightarrow 1^+}\pi(U_{\rm FP}(\SPACE{O});\VECTOR{\psi})=\pi([1,5,1,2]^\top;\VECTOR{\psi}),\ \forall\VECTOR{\psi}.
\end{aligned}
\end{equation}
The policy discontinuity introduced by FA representations may bring difficulties to policy learning since the expected policy $F_{\rm PI}(\SPACE{X}, x_{\rm else})$ should be continuous w.r.t. to each element in $\mathcal{X}$. Besides, it is usually difficult to design a proper sorting rule for complex driving scenarios such as intersections.

To conclude, due to the permutation sensitivity, AP and FP methods suffer from high sample complexity and policy discontinuity respectively, which may result in poor policy learning accuracy.

\section{Encoding Sum and Concatenation State Representation}
\label{sec:ESC}
Both dimension sensitivity and permutation sensitivity will increase the policy learning difficulty and limit the applicability of T2E decision-making in different driving scenarios. In the last decade, permutation-invariance approximation methods have been extensively studied \cite{zaheer2017deepset, maron2020Deepset, sannai2019universalDeepset}. However, these methods are only applicable to (a) countable case where $x_i$ is from a finite set, or (b) uncountable case with fixed set size $M$ of $\mathcal{X}$ where $x_i$ is from a continuous space, but barely valid on uncountable case with a variable set size $M$. In this section, the existing permutation-invariance approximation theory is extended to the field of state representation in autonomous driving, and an encoding sum and concatenation (ESC) method is proposed to realize the fixed-dimensional and permutation invariant state representation of the observation set $\SPACE{O}$.

\subsection{State Representation}
\label{subsec.PI_module}
As shown in Fig. \ref{fig.PI_method},  the mathematical description of the proposed ESC state representation is
\begin{equation}
\label{eq.pi_state}
    \VECTOR{s}=U_{\rm ESC}(\SPACE{O};\VECTOR{\phi})=\left[\begin{array}{c}\VECTOR{x}_{\rm set}\\\VECTOR{x}_{\rm else}\end{array}\right]=\left[\begin{array}{c}\sum_{\VECTOR{x}\in\SPACE{X}}h(\VECTOR{x};\VECTOR{\phi})\\\VECTOR{x}_{\rm else}\end{array}\right],
\end{equation}
where $h(\VECTOR{x};\VECTOR{\phi}):\mathbb{R}^{d_1}\rightarrow\mathbb{R}^{d_3}$ is the feature NN with parameters $\VECTOR{\phi}$ and $d_3$ is the output dimension. Different from $U_{\rm AP}$ and $U_{\rm FP}$, the ESC mapping $U_{\rm ESC}(\SPACE{O};\VECTOR{\phi})$ is a parameterized function. ESC first encodes each $\VECTOR{x}$ in the set $\SPACE{X}$ into the corresponding encoding vector $\VECTOR{x}_{\rm encode}\in\mathbb{R}^{d_3}$, i.e.,
\begin{equation}
\VECTOR{x}_{\rm encode}=h(\VECTOR{x};\VECTOR{\phi}).
\end{equation}
Then, we obtain the representation vector $x_{\rm set}$ of the surrounding vehicles set by summing the encoding vector of each surrounding vehicle
\begin{equation}
\label{eq04:PI_encode}
\VECTOR{x}_{\rm set}=\sum_{\VECTOR{x}\in\SPACE{X}}h(\VECTOR{x};\VECTOR{\phi}).
\end{equation}
\begin{figure}[!htb]
    \captionsetup{justification=raggedright, 
                  singlelinecheck=false, font=small}
    \centering{\includegraphics[width=0.4\textwidth]{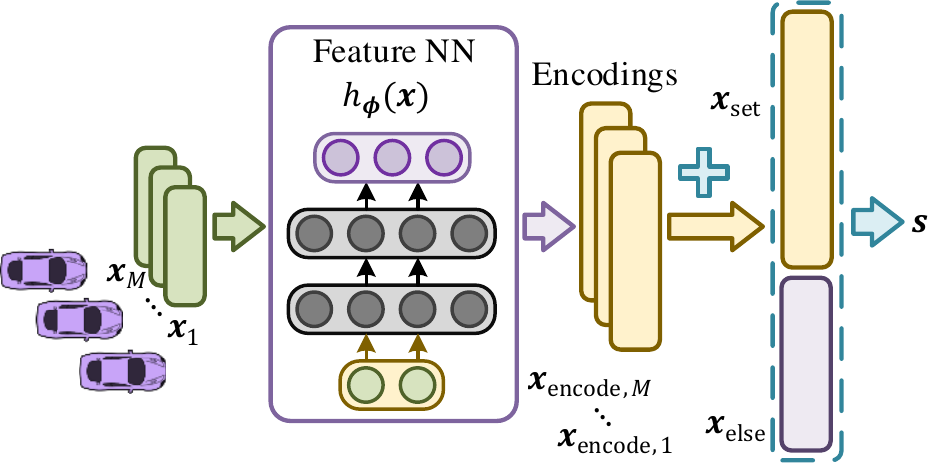}}
    \caption{ESC state representation.}
    \label{fig.PI_method}
\end{figure}

From \eqref{eq04:PI_encode}, it is clear that $\text{dim}(\VECTOR{x}_{\rm set})=\text{dim}(h_{\VECTOR{\phi}})=d_3$ for $\forall M \in [1,N]\cap\mathbb{N}$. In other words, $\VECTOR{x}_{\rm set}$ is fixed-dimensional. Furthermore, the summation operator in \eqref{eq04:PI_encode} is permutation invariant w.r.t. $\SPACE{X}$. Thus, $U_{\rm ESC}(\SPACE{O};\VECTOR{\phi})=[\VECTOR{x}_{\rm set}^\top,\VECTOR{x}_{\rm else}^\top]^\top$ is a fixed-dimensional and permutation invariant state representation of observation $\mathcal{O}$. Note that if $M=0$, one can add a virtual surrounding vehicle that is far away from the ego vehicle, which brings no effect on the decision-making.

By taking $U_{\rm ESC}(\SPACE{O};\VECTOR{\phi})$ as the inputs of $\pi_{\VECTOR{\psi}}$, the policy function can be expressed as 
\begin{equation}
\label{eq04:PI_formula}
\begin{aligned}
\pi(U_{\rm ESC}(\SPACE{O};\VECTOR{\phi});\VECTOR{\psi})=\pi(\sum_{\VECTOR{x}\in\SPACE{X}}h(\VECTOR{x};\VECTOR{\phi}),\VECTOR{x}_{\rm else};\VECTOR{\psi}),
\end{aligned}
\end{equation}
where $\pi(U_{\rm ESC}(\SPACE{O};\VECTOR{\phi});\VECTOR{\psi})$ is permutation invariant w.r.t. set $\SPACE{X}$. As shown in Fig. \ref{fig.PI_module}, the policy falls into two layers: (1) an ESC representation layer and (2) an decision layer. In the sequel, we refer to the policy function $\pi(U_{\rm ESC}(\SPACE{O};\VECTOR{\phi});\VECTOR{\psi})$ based on the ESC representations as the ESC policy. 

\begin{figure}[!htb]
    \captionsetup{justification=raggedright, 
                  singlelinecheck=false, font=small}
    \centering{\includegraphics[width=0.5\textwidth]{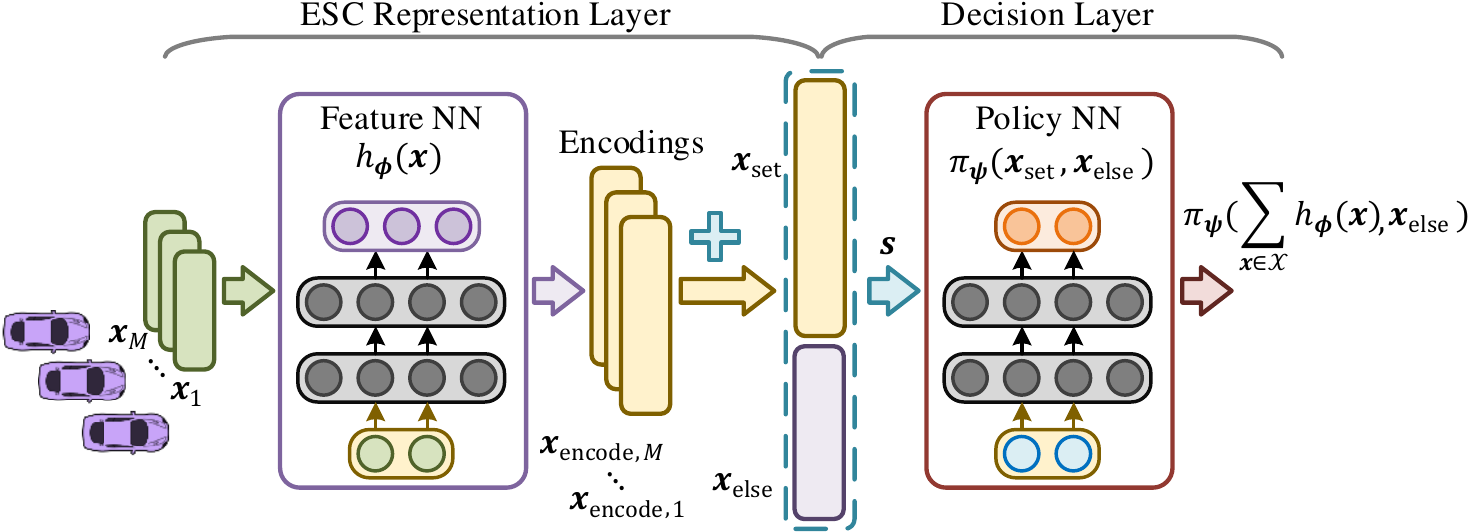}}
    \caption{Policy based on ESC state representation.}
    \label{fig.PI_module}
\end{figure}

\subsection{Injection and Optimality Analysis}
In addition to the fixed dimension and permutation invariance properties, to ensure the existence of $\psi^*$ and $\phi^*$, such that
\begin{equation}
\label{eq.goal}
\begin{aligned}
\pi(
U_{\rm ESC}(\SPACE{X},x_{\rm else}; \phi^*);\psi^*)&\approx F_{\rm PI}(\SPACE{X},x_{\rm else}), \\
& \forall\SPACE{X}\in \overline{\SPACE{X}},\forall x_{\rm else}\in\mathbb{R}^{d_2},
\end{aligned}
\end{equation}
the ESC state representation or ESC policy also needs to be injective w.r.t. the surrounding vehicles set $\SPACE{X}$.
If $U_{\rm ESC}$ is an injective mapping, for any  $\SPACE{X}_1,\SPACE{X}_2\in\overline{\SPACE{X}}$ where $ \SPACE{X}_1\neq\SPACE{X}_2$, it holds that $\sum_{\VECTOR{x}\in\SPACE{X}_1}h(\VECTOR{x};\VECTOR{\phi})\neq\sum_{\VECTOR{x}\in\SPACE{X}_2}h(\VECTOR{x};\VECTOR{\phi})$ or $U_{\rm ESC}(\{\SPACE{X}_1,\VECTOR{x}_{\rm else}\};\VECTOR{\phi})\neq U_{\rm ESC}(\{\SPACE{X}_2,\VECTOR{x}_{\rm else}\};\VECTOR{\phi})$. In contrast, if it is non-injective, there exist $\SPACE{X}_1,\SPACE{X}_2\in\overline{\SPACE{X}}$ where $\SPACE{X}_1\neq\SPACE{X}_2$, such that $U_{\rm ESC}(\{\SPACE{X}_1,\VECTOR{x}_{\rm else}\};\VECTOR{\phi})= U_{\rm ESC}(\{\SPACE{X}_2,\VECTOR{x}_{\rm else}\};\VECTOR{\phi})$. This indicates two different driving scenarios correspond to the identical state representation, which leads to the same policy outputs, thus impairing driving safety. Therefore, it is crucial to make sure that there $\exists \phi^{\dagger}$ such that $U_{\rm ESC}(\SPACE{X},x_{\rm else}; \phi^{\dagger})$ is injective.

Before proving the injectivity of the proposed ESC method, the following two lemmas are needed.

\begin{lemma} 
\label{lemma.UAT}
(Universal Approximation Theorem \cite{Hornik1990Universal}). For any continuous function $F(x):\mathbb{R}^n\rightarrow\mathbb{R}^d$ on a compact set $\Omega$, there exists an over-parameterized NN (i.e., the number of hidden neurons is sufficiently large), which uniformly approximates $F(x)$ and its gradient to within arbitrarily small error $\epsilon \in \mathbb{R}_{+}$ on $\Omega$.
\end{lemma}

\begin{lemma}\label{lemma.sum-of-power}
    (Sum-of-power mapping\cite{zaheer2017deepset}).
    Let $\mathcal{Z}=\{z_1,\cdots,z_m\}$, where $z_i\in[0,1]$ for $i\in [1,m]$, and define a sum-of-power mapping $E_n$ as
    \begin{equation}
        E_n(\mathcal{Z})=\left[\begin{array}{c}\sum_{i=1}^m(z_i)^1\\\vdots\\\sum_{i=1}^m(z_i)^n\end{array}\right].
    \end{equation}
    The mapping $E_n(\mathcal{Z})\in \mathbb{R}^{n}$ is an injection (i.e. $\mathcal{Z}_1\neq \mathcal{Z}_2\rightarrow E_n(\mathcal{Z}_1)\neq E_n(\mathcal{Z}_2)$) if $n\ge m$.
\end{lemma}

Then, the main theorem is given as follows.

\begin{theorem}\label{theorem.PI}
 (Injectivity of the ESC State Representation). Let $\SPACE{O}=\{\SPACE{X},\VECTOR{x}_{\rm else}\}$, where $\VECTOR{x}_{\rm else}\in\mathbb{R}^{d_2}$ and $\mathcal{X}=\{x_1,x_2,\cdots,x_M\}$. Denote the space of $\SPACE{X}$ as $\overline{\SPACE{X}}$, where $\overline{\mathcal{X}}=\{\mathcal{X}|\mathcal{X}=\{x_1,\cdots,x_{M}\},x_i\in[c_{\rm min}, c_{\rm max}]^{d_1},i\le M,M\in[1,N]\cap\mathbb{N}\}$, in which $c_{\rm min}$ and $c_{\rm max}$ are the lower and upper bounds of all elements in $\forall\VECTOR{x_i}$, respectively. Noted that the size $M$ of the set $\mathcal{X}$ is variable. If the feature NN $h(\VECTOR{x};\VECTOR{\phi}):\mathbb{R}^{d_1}\rightarrow\mathbb{R}^{d_3}$ is over-parameterized with a linear output layer, and its output dimension $d_3\ge Nd_1+1$, there always $\exists \phi^{\dagger}$ such that the mapping $U_{\rm ESC}(\mathcal{O};\phi^{\dagger}): \overline{\mathcal{X}}\times \mathbb{R}^{d_2}\rightarrow \mathbb{R}^{d_3+d_2}$ in \eqref{eq.pi_state} is injective.
\end{theorem}

\begin{proof}
Let $\VECTOR{x}_i=[x_{i,1},\cdots,x_{i,d_1}]^{\top}$. We concatenate the $j$th element of each $\VECTOR{x}_i$ into the set $\VECTOR{X}_{j}=\{x_{1,j},\cdots,x_{M,j}\}$.  By normalizing $\VECTOR{X}_{j}$ using the min-max scaling method, for $\forall j\in[1,d_1]$, we will get
\begin{equation}
\VECTOR{X}_{{\rm norm},j}=\big\{\frac{x_{1,j}-c_{\rm min}}{c_{\rm max}-c_{\rm min}},\cdots,\frac{x_{M,j}-c_{\rm min}}{c_{\rm max}-c_{\rm min}}\big\}.
\end{equation}
According to Lemma \ref{lemma.sum-of-power}, when $n \ge M$, the sum-of-power mapping $E_n(\VECTOR{X}_{{\rm norm},j})$ expressed as 
\begin{equation}
\label{eq04:injective_mapping}
E_n(\VECTOR{X}_{{\rm norm},j})=\left[\begin{array}{c}\sum_{i=1}^M\big(\frac{x_{i,j}-c_{\rm min}}{c_{\rm max}-c_{\rm min}}\big)^1\\\vdots\\\sum_{i=1}^M\big(\frac{x_{i,j}-c_{\rm min}}{c_{\rm max}-c_{\rm min}}\big)^n\end{array}\right]
\end{equation}
is injective when $M$ is a fixed value. 

From \eqref{eq04:injective_mapping}, since $N\ge M$, the mapping $G$ defined as 
\begin{equation}
\label{eq.G_map}
G\left( \{\VECTOR{X}_{\rm{norm},1},\cdots,\VECTOR{X}_{\rm{norm},d_1}\}\right)=\left[\begin{array}{c}E_N(\VECTOR{X}_{\rm{norm},1})\\\vdots\\E_N(\VECTOR{X}_{\rm{norm},d_1})\\\sum_1^M1\end{array}\right]
\end{equation}
is also injective. In particular, the existence of item $\sum_1^M1$ makes the mapping $G$ also suitable for the case where the set size $M$ is variable.

Furthermore, according to Lemma \ref{lemma.UAT}, when $d_3\ge Nd_1+1$, there always $\exists \phi^{\dagger}$, such that
\begin{equation}
h(\VECTOR{x}_i;\VECTOR{\phi}^{\dagger})= \left[\begin{array}{c}\big(\frac{x_{i,1}-c_{\rm min}}{c_{\rm max}-c_{\rm min}}\big)^1\\\vdots\\\big(\frac{x_{i,1}-c_{\rm min}}{c_{\rm max}-c_{\rm min}}\big)^N\\\vdots \\ \big(\frac{x_{i,d_1}-c_{\rm min}}{c_{\rm max}-c_{\rm min}}\big)^1\\\vdots\\\big(\frac{x_{i,d_1}-c_{\rm min}}{c_{\rm max}-c_{\rm min}}\big)^N \\1\\\vdots \end{array}\right],\ \forall \VECTOR{x}_i\in[c_{\rm min},c_{\rm max}]^{d_1}.
\end{equation}
Then, it directly follows that
\begin{equation}
\VECTOR{x}_{\rm set}=\sum_{\VECTOR{x}\in\SPACE{X}}h(\VECTOR{x};\phi^{\dagger})=\left[\begin{array}{c}E_N(\VECTOR{X}_{\rm{norm},1})\\\vdots\\E_N(\VECTOR{X}_{\rm{norm},d_1})\\\sum_1^M1\\\vdots\end{array}\right].
\end{equation}
Similar to \eqref{eq.G_map},  $\sum_{\VECTOR{x}\in\SPACE{X}}h(\VECTOR{x};\phi^{\dagger}):\overline{\mathcal{X}}\rightarrow \mathbb{R}^{d_3}$ is injective, which means $U_{\rm PI}(\SPACE{O};\phi^{\dagger}): \overline{\mathcal{X}}\times \mathbb{R}^{d_2}\rightarrow \mathbb{R}^{d_3+d_2}$ is also injective. 
\end{proof}

Next, we will analyze the optimality of the ESC representation and the ESC policy.

\begin{lemma}
\label{lemma.global_min} (Global Minima of Over-Parameterized NNs \cite{allen2018convergence,du2019overconverge}).
Consider the following optimization problem
\begin{equation}
\nonumber
\min_{\psi}\mathcal{L}(\psi) = \mathop{\mathbb{E}}_{X_i\in\mathcal{B}}\Big\{\frac{1}{2}(\mathcal{F}(X_i;\psi)-Y_i)^2\Big\},
\end{equation}
where $X_i\in \mathbb{R}^n$ is the training input, $Y_i \in \mathbb{R}^d$ is the associated label, $\mathcal{B}=\{(X_1,Y_1),(X_2,Y_2),\hdots\}$ is the dataset, $\psi$ is the parameter to be optimized, and $\mathcal{F}:\mathbb{R}^n \rightarrow \mathbb{R}^d$ is an NN. If the NN $\mathcal{F}(X;\psi)$ is over-parameterized, simple algorithms such as gradient descent (GD) or stochastic GD (SGD) can find global minima on the training objective $\mathcal{L}(\psi)$ in polynomial time, as long as the dataset $\mathcal{B}$ is non-degenerate. The dataset is non-degenerate if the same inputs $X_1=X_2$ have the same labels $Y_1=Y_2$.
\end{lemma}

\begin{theorem}
\label{theorem.optimality}
Given any continuous function operating on a the set $\mathcal{O}$, i.e., $F_{\rm PI}:\overline{\mathcal{X}} \times \mathbb{R}^{d_2}\rightarrow\mathcal{Y}$, which is permutation invariant to the elements in $\mathcal{X}$. Suppose $\mathcal{O}$, $\overline{\mathcal{X}}$, and $\mathcal{X}$ as described in Theorem \ref{theorem.PI}. If the feature NN $h(\VECTOR{x};\VECTOR{\phi}):\mathbb{R}^{d_1}\rightarrow\mathbb{R}^{d_3}$ and policy NN $\pi(U_{\rm ESC}(\SPACE{O};\VECTOR{\phi});\VECTOR{\psi}):\mathbb{R}^{d_2+d_3}\rightarrow\mathcal{Y}$ are both over-parameterized, and $d_3\ge Nd_1+1$, we can find $\phi^*$ and $\psi^*$ which make \eqref{eq.goal} hold by directly minimizing $L(\psi,\phi)$ using optimization methods such as GD and SGD, where
\begin{equation}
\label{eq.objective_PI}
\begin{aligned}
L(\psi,\phi)=\Exp_{\substack{\SPACE{X}\in\overline{\SPACE{X}},\\x_{\rm else}\in\mathbb{R}^{d_2}}}\big(\pi(U_{\rm ESC}(\SPACE{O};\phi);\VECTOR{\psi})-F_{\rm PI}(\SPACE{X},x_{\rm else})\big)^2.
\end{aligned}
\end{equation}
\end{theorem}

\begin{proof}
From Theorem \ref{theorem.PI}, there $\exists \phi^{\dagger}$ such that $U_{\rm ESC}(\mathcal{O};\phi^{\dagger}): \overline{\mathcal{X}}\times \mathbb{R}^{d_2}\rightarrow \mathbb{R}^{d_3+d_2}$ in \eqref{eq.pi_state} is injective. Therefore, from Lemma \ref{lemma.UAT}, one has
\begin{equation}
\min_{\psi} L(\psi,\phi^{\dagger})\approx 0.
\end{equation}
In other words, there exists a pair of $\psi$ and $\phi$, which makes $\pi(U_{\rm ESC}(\SPACE{O};\VECTOR{\phi});\VECTOR{\psi})$ approximate $F_{\rm PI}$ arbitrarily close. Although the nearly optimal parameters may not be unique, according to Lemma \ref{lemma.global_min}, we can find a pair of $\phi^*$ and $\psi^*$ which makes \eqref{eq.goal} hold by directly minimizing $L(\psi,\phi)$ using optimization methods such as GD and SGD. 
\end{proof}
The existence of the injective mapping given in Theorem \ref{theorem.PI} can ensure that  $\min_{\{\psi,\phi\}}L(\psi,\phi)\approx 0$ holds. Although the solution $U_{\rm ESC}(\SPACE{O};\phi^*)$ found in Theorem \ref{theorem.optimality} may be non-injective, $\pi(U_{\rm ESC}(\SPACE{O};\phi^*);\psi^*)$ still approximates the target function $F_{\rm PI}$ arbitrarily close. 

\begin{remark}
The feature NN $h_{\VECTOR{\phi}}$ is only related to the space $\overline{\mathcal{X}}$ of set $\mathcal{X}$, but is independent of function $F_{\rm PI}$. This indicates that for any different continuous permutation invariant functions $F_{\rm PI,1}$ and $F_{\rm PI,2}$ operating on set $\mathcal{O}$, for the same injective mapping  $U_{\rm ESC}(\mathcal{O};\phi^{\dagger})$, there exist $\psi_1$ and $\psi_2$ assuring
\begin{equation}
    \nonumber
        \pi(U_{\rm PI}(\SPACE{O};\VECTOR{\phi}^{\dagger});\VECTOR{\psi}_1)\approx F_{\rm PI,1}(\SPACE{X},\VECTOR{x}_{\rm else}) 
\end{equation}    
and
\begin{equation}
    \nonumber
        \pi(U_{\rm PI}(\SPACE{O};\VECTOR{\phi}^{\dagger});\VECTOR{\psi}_2)\approx F_{\rm PI,2}(\SPACE{X},\VECTOR{x}_{\rm else})
\end{equation}
for $\forall\SPACE{X}\in\overline{\SPACE{X}}$ and $\forall\VECTOR{x}_{\rm else}\in\mathbb{R}^{d_2}$, respectively.
\end{remark}
\begin{remark}
In Theorem \ref{theorem.PI} and \ref{theorem.optimality}, we require the feature of each surrounding vehicle to satisfy $x \in[c_{\rm min}, c_{\rm max}]^{d_1}$,  where $c_{\rm min}$ and $c_{\rm max}$ are the lower and upper bounds of all elements in $x$. We know that for actual autonomous driving applications, the range of different indicators, such as velocity or heading angle, may be different. However, by utilizing some normalization methods, such as min-max feature scaling, we can easily normalize each element to the same range.
\end{remark}

\section{Experimental Verification}
\label{sec:exp}
\begin{table*}[!htb]
\centering
\caption{Benchmarks}
\label{tab04:task}
\begin{tabular}{cc}
\toprule[1.5pt]
 No.& Expected permutation invariant policy functions\\
\midrule[1pt]
1&$F_{\rm PI}(\SPACE{X},\VECTOR{x}_{\rm else})={\rm mean}(\VECTOR{x}_{\rm else})-0.2\min([\|\VECTOR{x}_1\|_3,\cdots,\|\VECTOR{x}_M\|_3])+0.4{\rm mean}([\|\VECTOR{x}_1\|_1,\cdots,\|\VECTOR{x}_M\|_1])\times\max([\|\VECTOR{x}_1\|_2,\cdots,\|\VECTOR{x}_M\|_2]) $\\
\midrule[1pt]
2&$F_{\rm PI}(\SPACE{X},\VECTOR{x}_{\rm else})=0.5\min(\VECTOR{x}_{\rm else})\times\max([\max(\VECTOR{x}_1),\cdots,\max(\VECTOR{x}_M)])\times\min([\|\VECTOR{x}_1\|_4,\cdots,\|\VECTOR{x}_M\|_4])$\\
\midrule[1pt]
3&$F_{\rm PI}(\SPACE{X},\VECTOR{x}_{\rm else})=0.2\|\VECTOR{x}_{\rm else}\|_3+2{\rm mean}([\|\VECTOR{x}_1\|_1,\cdots,\|\VECTOR{x}_M\|_1])\times{\rm mean}([\max(\VECTOR{x}_1),\cdots,\max(\VECTOR{x}_M)])$\\
\midrule[1pt]
4&$F_{\rm PI}(\SPACE{X},\VECTOR{x}_{\rm else})=5\|\VECTOR{x}_{\rm else}\|_2\times \Big\|\Big[\frac{\min(\VECTOR{x}_1)}{\|\VECTOR{x}_{1}\|_2+0.1},\cdots,\frac{\min(\VECTOR{x}_M)}{\|\VECTOR{x}_{M}\|_2+0.1}\Big]\Big\|_4$ \\
\midrule[1pt]
5&$F_{\rm PI}(\SPACE{X},\VECTOR{x}_{\rm else})=10\|\VECTOR{x}_{\rm else}\|_4\times{\rm mean} \Big[\frac{{\rm mean}(\VECTOR{x}_1)\max(\VECTOR{x}_1)}{\|\VECTOR{x}_{1}\|_4+0.1},\cdots,\frac{{\rm mean}(\VECTOR{x}_M)\max(\VECTOR{x}_M)}{\|\VECTOR{x}_{M}\|_4+0.1}\Big]$ \\
\midrule[1pt]
6&$F_{\rm PI}(\SPACE{X},\VECTOR{x}_{\rm else})=8\|\VECTOR{x}_{\rm else}\|_2\times \max\Big[\frac{{\rm mean}(\VECTOR{x}_1)\|\VECTOR{x}_{1}\|_3}{\|\VECTOR{x}_{1}\|_2+0.1},\cdots,\frac{{\rm mean}(\VECTOR{x}_M)\|\VECTOR{x}_{M}\|_3}{\|\VECTOR{x}_{M}\|_2+0.1}\Big]$ \\
\bottomrule[1.5pt]
\end{tabular}
\end{table*}

This section validates the effectiveness of the proposed ESC method in a general policy learning task based on supervised learning. We take AP and FP representation methods as baselines.

\subsection{Experiments Design}
We set the dimension of $\VECTOR{x}$ to $d_1=5$, and each element of $x$ is bounded by $c_{\rm{min}}=-5$ and $c_{\rm{max}}=5$, i.e., $\VECTOR{x}\in [-5,5]^5$. Similarly, we set $\VECTOR{x}_{\rm else}\in [-5,5]^{10}$. We assume that the maximum size of set $\mathcal{X}$ is $N=20$, i.e., $M\in[1,20]$. Based on these settings, we construct six expected policy functions in Table \ref{tab04:task} as benchmarks. Noted that $\min(\VECTOR{z})$, $\max(\VECTOR{z})$, ${\rm mean}(\VECTOR{z})$ in Table \ref{tab04:task} represent taking the minimum, maximum, and mean value of elements in $\VECTOR{z}$, respectively, and $\|\VECTOR{z}\|_p$ denotes the $p$-norm of $\VECTOR{z}$. $\min(\VECTOR{z})$, $\max(\VECTOR{z})$, and ${\rm mean}(\VECTOR{z})$ are three typical permutation invariant operators. For example, the vehicle closest to the ego vehicle is usually an important reference for decision-making. The combination of these three operators can form many representative nonlinear permutation invariant functions. 

We will learn a policy to approximate each benchmark using different state representation methods. Then the performance of the ESC method can be evaluated by comparing the policy approximation accuracy of different representations. As shown in Table \ref{tab04:experiment_plan}, according to the set size of $\mathcal{X}$, the experiment for each benchmark is divided into five cases, $M=5$, $M=10$, $M=15$, $M=20$ and $M\in[1,20]$. In particular, only ESC is applicable to variable size set $\mathcal{X}$, that is, case 5.

\begin{table}[!htb]
\centering
\caption{Five experimental settings of each benchmark}
\label{tab04:experiment_plan}
\begin{tabular}{ccc}
\toprule[1.5pt]
Case & Set Size $M$ & Representation methods\\
\midrule[1pt]
1 & $5$ & 1) ESC; 2) FP; 3) AP\\
2 & $10$ & 1) ESC; 2) FP; 3) AP\\
3 & $15$ & 1) ESC; 2) FP; 3) AP\\
4 & $20$ & 1) ESC; 2) FP; 3) AP\\
5 & $M\in[1,20]$ & ESC\\
\bottomrule[1.5pt]
\end{tabular}
\end{table}

For each case of each benchmark, we randomly generated a training set $\SPACE{S}_{\rm train}$ containing one million samples and a test set $\SPACE{S}_{\rm test}$ containing 2048 samples. The $j$th sample in $\SPACE{S}_{\rm train}$ or $\SPACE{S}_{\rm test}$ is denoted as $\{\{\SPACE{X}_j,{\VECTOR{x}_{\rm else}}_j\},y_j\}$, where $\SPACE{X}_j$ and ${\VECTOR{x}_{\rm else}}_j$ are sampled uniformly within their space, and $y_j=F_{\rm PI}(\SPACE{X}_j,{\VECTOR{x}_{\rm else}}_j)$. Given $\SPACE{S}_{\rm train}$, the policy NN $\pi_{\psi}$ (and the feature NN $h_{\phi}$ for ESC) based on the AP, FP, and ESC are optimized by directly minimizing \eqref{eq.all_permutaion}, \eqref{eq.order_permutation} and \eqref{eq.objective_PI}, respectively. For the FP method, the pre-designed order $o$ arranges the elements of $\mathcal{X}$ in increasing order according to the first element of $x_i$. If the first element is equal,  we will compare the second element, and so on. 

\begin{remark}
According to Theorem \ref{theorem.PI} and \ref{theorem.optimality}, the proposed ESC method is a general state representation method which is suitable for the uncountable case where $x_i$ comes from a continuous space with a variable set size $M$. It can be applicable in many fields, such as UAVs control and autonomous driving. The experiments provided in this section mainly focus on the evaluation of state representation ability of the proposed method for general permutation-invariance functions. Therefore, the experiment is not designed based on specific driving tasks. The combination of the ESC representation and policy learning methods such as RL, and their application in autonomous driving will be studied in the future.
\end{remark}

\subsection{Training Details}

For the ESC method, we  use  a  fully  connected NN  with  five hidden  layers,  consisting  of  256  units  per layer,  with  Gaussian  Error  Linear  Units  (GELU) as activation functions for each layer \cite{hendrycks2016gelu}, for both  feature NN and policy NN (See Figure \ref{fig04:PI-NN}).  The output layer of each NN is linear. According to Theorem \ref{theorem.PI}, the output dimension $d_3$ of $h_{\phi}$ should satisfy that $d_3\ge Nd_1+1=101$, so we set $d_3=101$. 

\begin{figure}[!htb]
    \centering
\captionsetup{singlelinecheck = false,labelsep=period, font=small}
\captionsetup[subfigure]{justification=centering}
        \subfloat[ ]{\label{fig04:PI-NN}\includegraphics[width=0.425\textwidth]{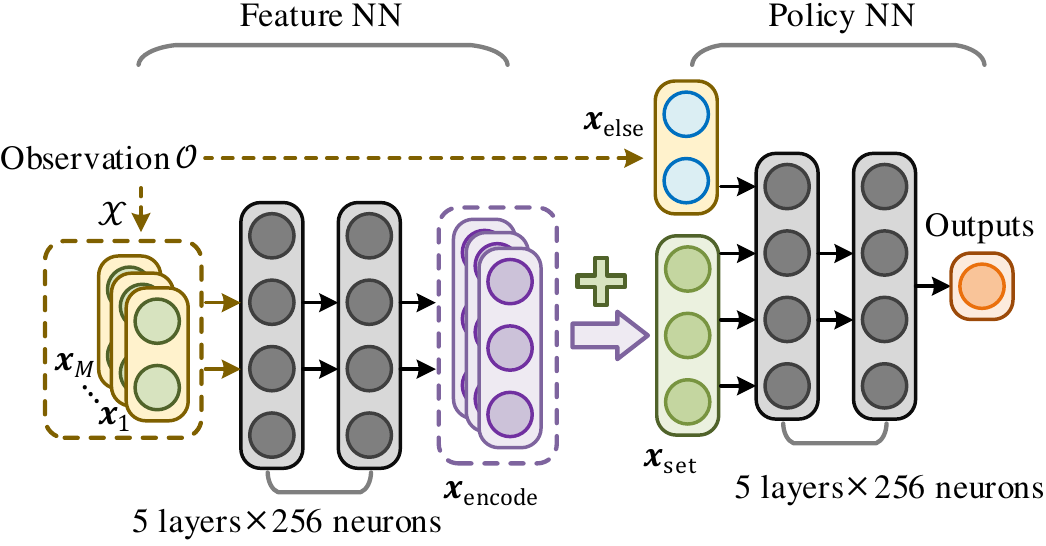}}\\
        \subfloat[]{\label{fig04:PS-NN}\includegraphics[width=0.35\textwidth]{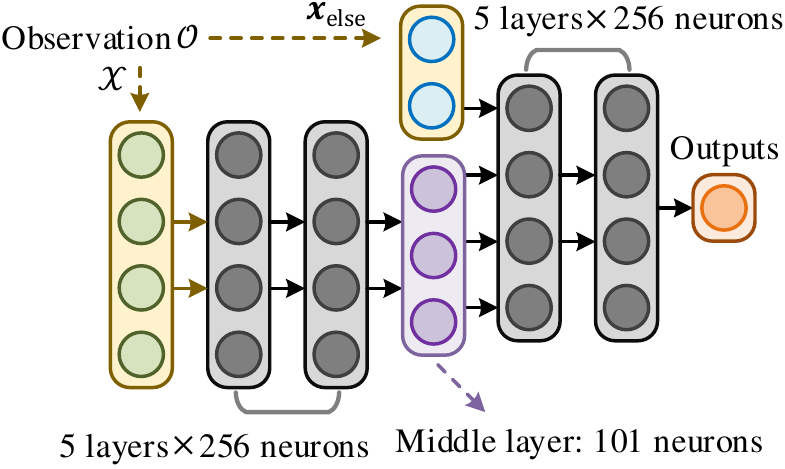}}
    \caption{NN architecture. (a) The NN architecture of the ESC method. (b) The NN architecture of AP and FP methods.}
    \label{fig04:NN_architecture}
\end{figure}

Unlike the ESC method containing two NNs, AP and FP only need to learn a policy NN. To avoid the influence of different NN architectures on policy learning accuracy, the policy NN for these two methods is designed as shown in Figure \ref{fig04:PS-NN}. This architecture comprises 11 hidden layers, in which each layer contains 256 units with GELU activations, except for the middle layer (i.e., the $6$th layer). The middle layer is a linear layer containing 101 units, which is equal to the output dimension of $h_{\phi}$. The input dimension is $5M$, which is related to the set size of $\mathcal{X}$. Therefore, the approximation structures in Figure \ref{fig04:PI-NN} and \ref{fig04:PS-NN} have the same number of hidden layers and neurons. In particular, when $M=1$, these two architectures are identical. This design will greatly reduce the impact of network structure differences on learning accuracy. By guaranteeing the similarity of approximation architectures, we can effectively evaluate the effects of different state representation methods.

For all representation methods, we adopt Adam \cite{Diederik2015Adam} to update NNs where the decay rate of first- and second-order moments are $0.9$ and $0.999$, respectively. The batchsize is 512 and the learning rate is $8\times10^{-5}$.

\subsection{Results Analysis}
We train 5 different runs of each representation method with different random seeds, and evaluate the policy learning accuracy by calculating the Root Mean Square Error (RMSE) based on $\mathcal{S}_{\rm test}$. The training curves of benchmark 1 are shown in Fig. \ref{fig04:result_task1}. In addition to the cases with fixed-size sets (case 1-4 in Table \ref{tab04:experiment_plan}), we also train an ESC policy based on the samples from the variable-size set (case 5). The learned ESC policy based on case 5 is evaluated when $M=5$, $M=10$, $M=15$ and $M=20$, respectively, shown as the blue solid lines in Fig. \ref{fig04:result_task1}. 

\begin{figure}[!htb]
\centering
\captionsetup{singlelinecheck = false,labelsep=period, font=small}
\captionsetup[subfigure]{justification=centering}
        \subfloat[]{\label{fig04:M=5}\includegraphics[width=0.24\textwidth]{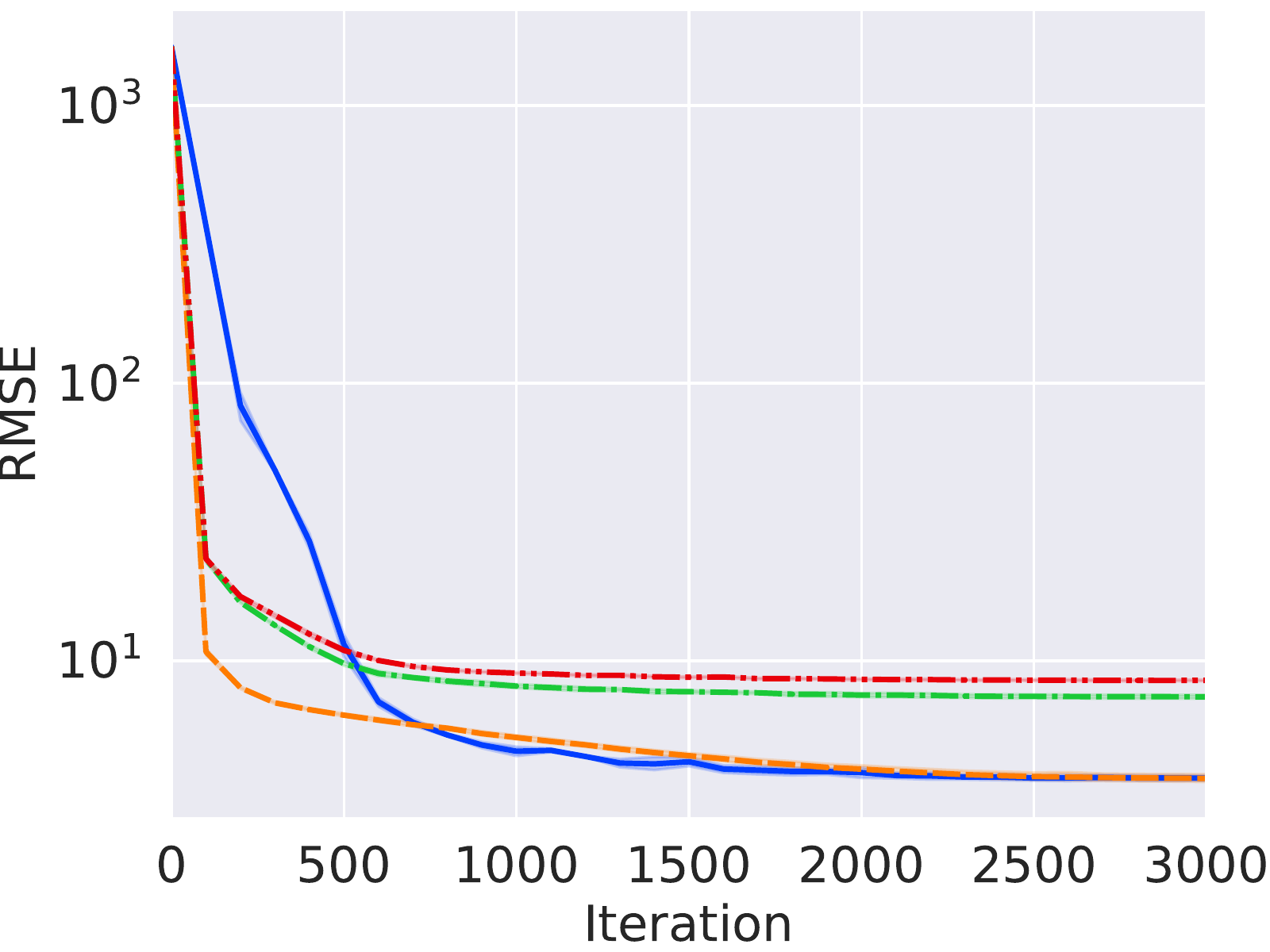}}
        \subfloat[$M=10$]{\label{fig04:M=10}\includegraphics[width=0.24\textwidth]{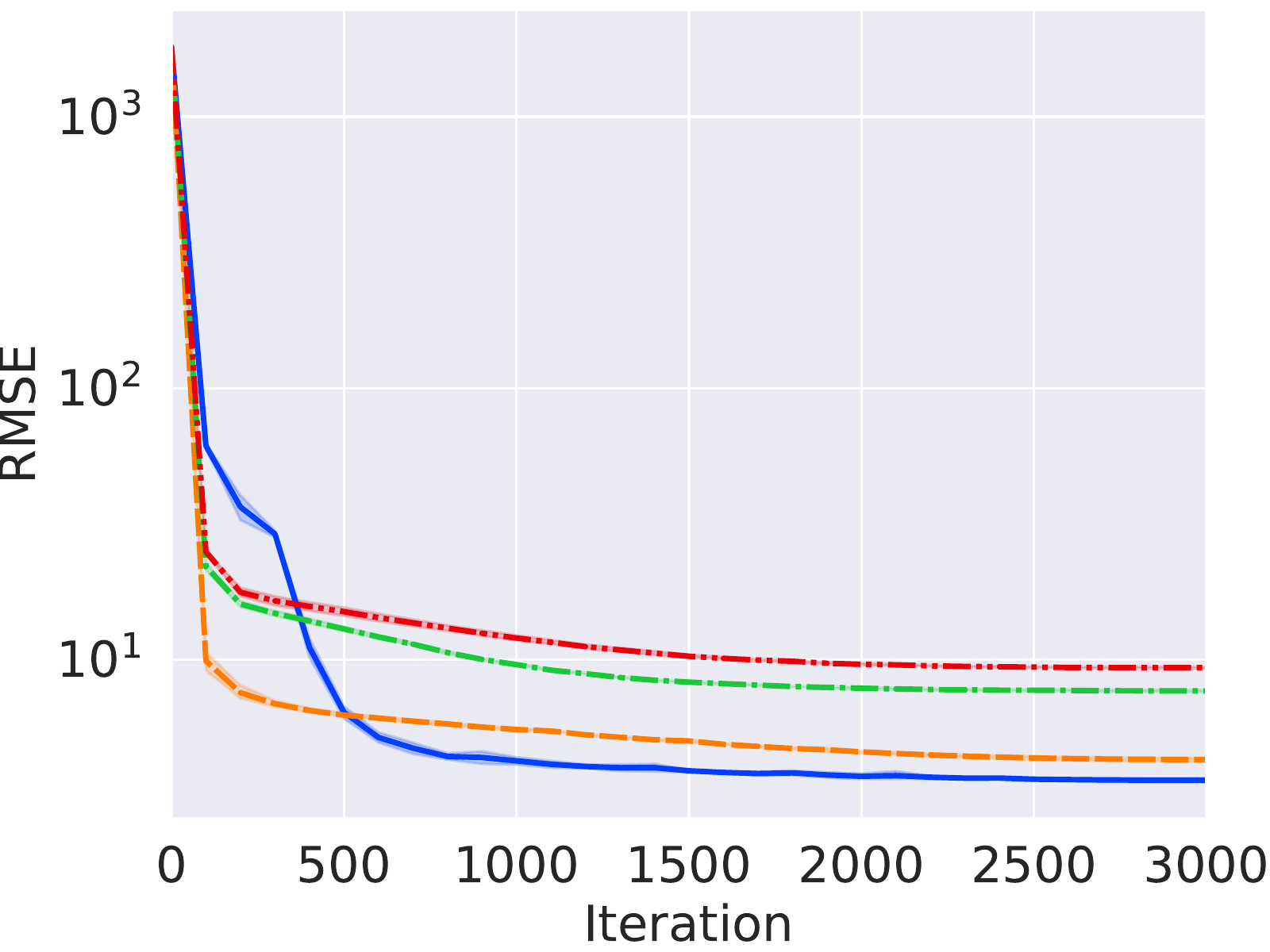}} \\
        \subfloat[$M=15$]{\label{fig04:M=15}\includegraphics[width=0.24\textwidth]{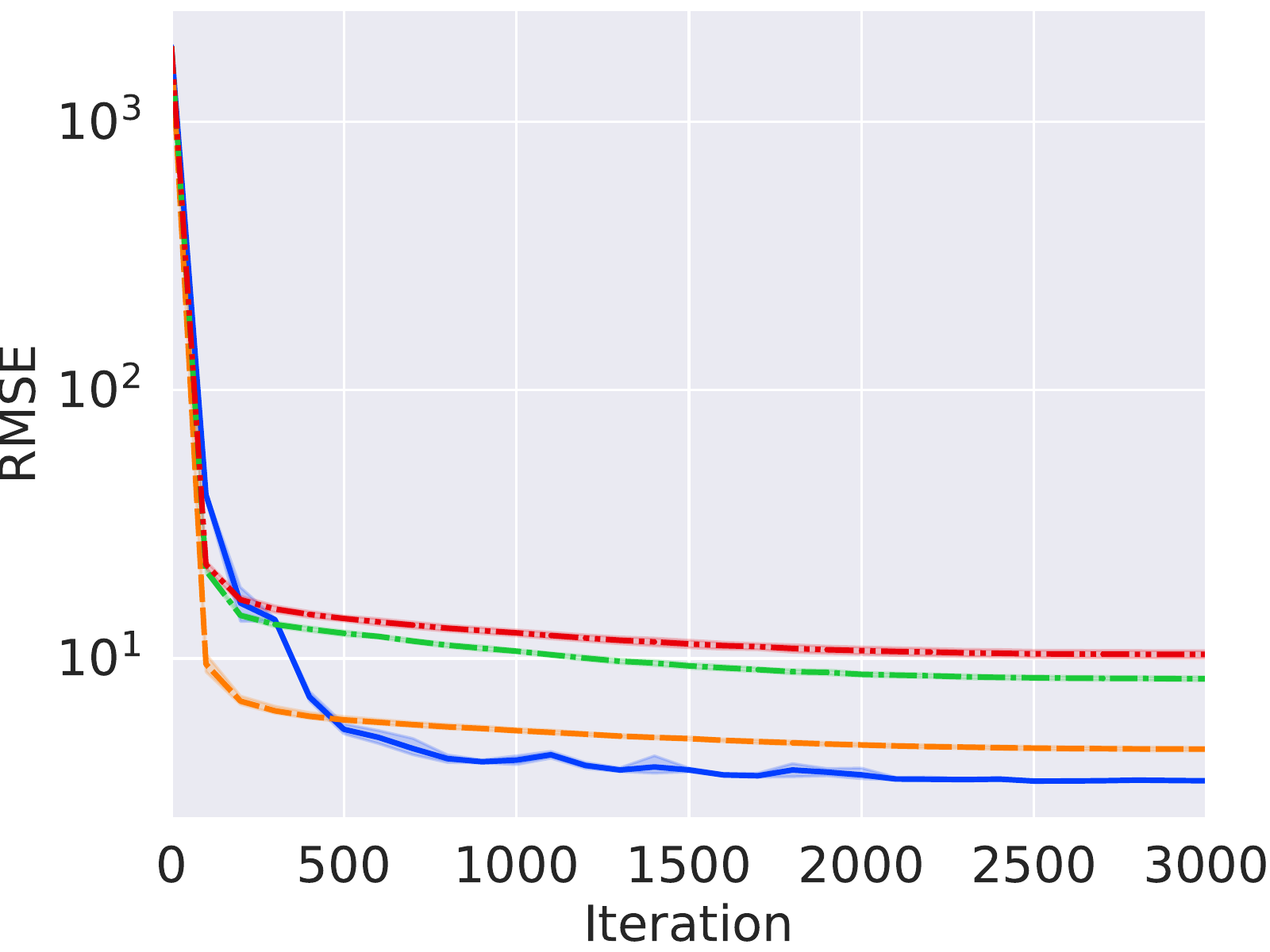}}
        \subfloat[$M=20$]{\label{fig04:M=20}\includegraphics[width=0.24\textwidth]{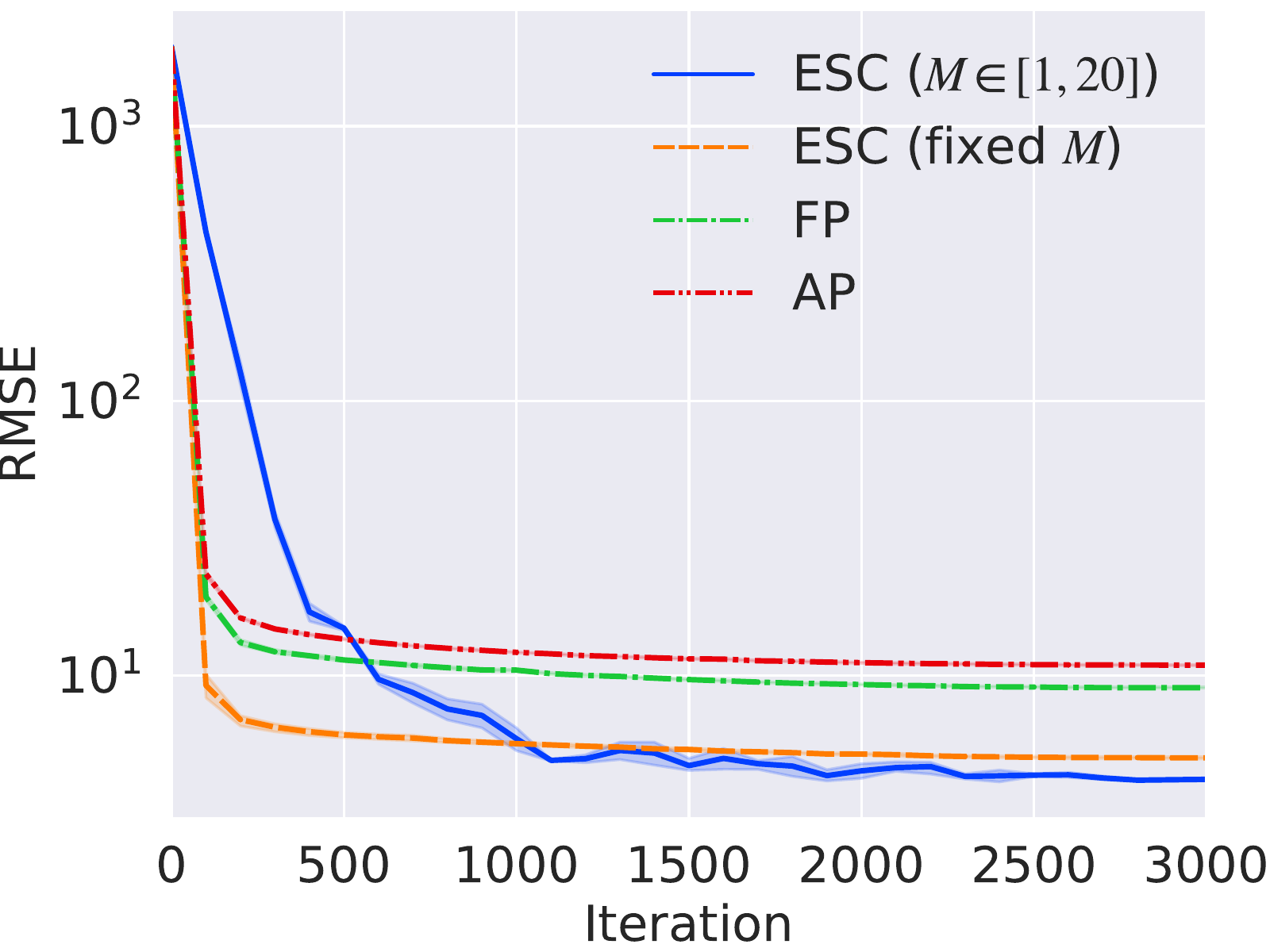}}
        \caption{Training curves of benchmark 1. The solid lines correspond to the mean RMSE and the shaded regions correspond to 95\% confidence interval over 5 runs. (a) $M=5$. (b) $M=10$. (c) $M=15$. (d) $M=20$. }
    \label{fig04:result_task1}
\end{figure}

\begin{table*}[!htb]
\centering
\caption{The average final RMSE on $\mathcal{S_{\rm test}}$ after training for $3000$ iterations. The minimum RMSE for each task is bolded. $\pm$ corresponds to a single standard deviation over 5 runs.}
\label{tab04:loss_comparison}
\begin{tabular}{cccccc}
\toprule[1.5pt]
Benchmark&Numbers of surrounding vehicles& ESC ($M\in[1,N]\cap\mathbb{N}$)& ESC (fixed $M$) &FP &AP \\
\midrule[1pt]
1& $M=5$ & 3.78$\pm$0.1 & \textbf{3.77$\pm$0.15} & 7.42$\pm$0.13 & 8.5$\pm$0.07 \\
& $M=10$ & \textbf{3.6$\pm$0.06} & 4.29$\pm$0.08 & 7.68$\pm$0.07 & 9.35$\pm$0.06 \\
& $M=15$ & \textbf{3.51$\pm$0.02} & 4.6$\pm$0.08 & 8.42$\pm$0.17 & 10.36$\pm$0.41 \\
& $M=20$ & \textbf{4.19$\pm$0.05} & 5.02$\pm$0.06 & 9.04$\pm$0.06 & 10.93$\pm$0.06 \\
\midrule[1.0pt]
2& $M=5$ & 36.87$\pm$0.29 & \textbf{30.69$\pm$0.36} & 53.63$\pm$0.01 & 55.83$\pm$0.02 \\
& $M=10$ & 31.83$\pm$0.57 & \textbf{27.76$\pm$0.15} & 56.42$\pm$0.05 & 60.14$\pm$0.36 \\
& $M=15$ & 30.15$\pm$0.67 & \textbf{29.97$\pm$1.05} & 54.18$\pm$0.41 & 56.25$\pm$0.12 \\
& $M=20$ & \textbf{32.6$\pm$0.94} & 33.9$\pm$1.08 & 51.58$\pm$0.19 & 53.56$\pm$0.46 \\
\midrule[1.0pt]
3& $M=5$ & 12.46$\pm$0.73 & \textbf{10.98$\pm$0.19} & 42.31$\pm$1.11 & 57.56$\pm$1.21 \\
& $M=10$ & \textbf{5.56$\pm$0.19} & 7.62$\pm$0.11 & 33.09$\pm$1.05 & 43.82$\pm$0.92 \\
& $M=15$ & \textbf{3.82$\pm$0.17} & 6.13$\pm$0.16 & 29.59$\pm$0.23 & 44.0$\pm$0.34 \\
& $M=20$ & 6.77$\pm$0.62 & \textbf{5.81$\pm$0.34} & 31.6$\pm$0.87 & 44.94$\pm$1.43 \\
\midrule[1.0pt]
4& $M=5$ & 5.96$\pm$0.17 & \textbf{4.42$\pm$0.36} & 10.82$\pm$0.14 & 12.28$\pm$0.07 \\
& $M=10$ & \textbf{4.33$\pm$0.14} & 4.7$\pm$0.15 & 9.4$\pm$0.21 & 10.19$\pm$0.05 \\
& $M=15$ & \textbf{3.8$\pm$0.08} & 4.46$\pm$0.26 & 8.3$\pm$0.37 & 9.19$\pm$0.1 \\
& $M=20$ & \textbf{3.89$\pm$0.11} & 4.72$\pm$0.07 & 8.07$\pm$0.2 & 8.53$\pm$0.08 \\
\midrule[1.0pt]
5& $M=5$ & 5.57$\pm$0.08 & \textbf{3.95$\pm$0.17} & 18.59$\pm$0.31 & 24.79$\pm$0.47 \\
& $M=10$ & 2.88$\pm$0.08 & \textbf{2.39$\pm$0.07} & 16.93$\pm$0.24 & 25.96$\pm$0.35 \\
& $M=15$ & 2.2$\pm$0.06 & \textbf{1.89$\pm$0.1} & 14.82$\pm$0.3 & 24.16$\pm$0.46 \\
& $M=20$ & 2.1$\pm$0.07 & \textbf{1.47$\pm$0.03} & 13.72$\pm$0.21 & 23.26$\pm$0.39 \\
\midrule[1.0pt]
6& $M=5$ & \textbf{40.88$\pm$1.33 }& 43.02$\pm$3.11 & 66.01$\pm$1.78 & 64.97$\pm$4.91 \\
& $M=10$ & \textbf{35.52$\pm$0.28} & 42.9$\pm$2.17 & 219.57$\pm$16.9 & 355.9$\pm$6.21 \\
& $M=15$ & \textbf{43.59$\pm$1.28} & 56.56$\pm$0.87 & 349.63$\pm$3.82 & 679.42$\pm$25.02 \\
& $M=20$ & \textbf{59.64$\pm$1.56} & 62.02$\pm$7.02 & 508.31$\pm$17.91 & 832.74$\pm$5.49 \\
\bottomrule[1.5pt]
\end{tabular}
\end{table*}

\begin{figure*}[!htb]
\centering
\captionsetup{singlelinecheck = false,labelsep=period, font=small}
\captionsetup[subfigure]{justification=centering}
        \subfloat[]{\label{fig04:task1}\includegraphics[width=0.30\textwidth]{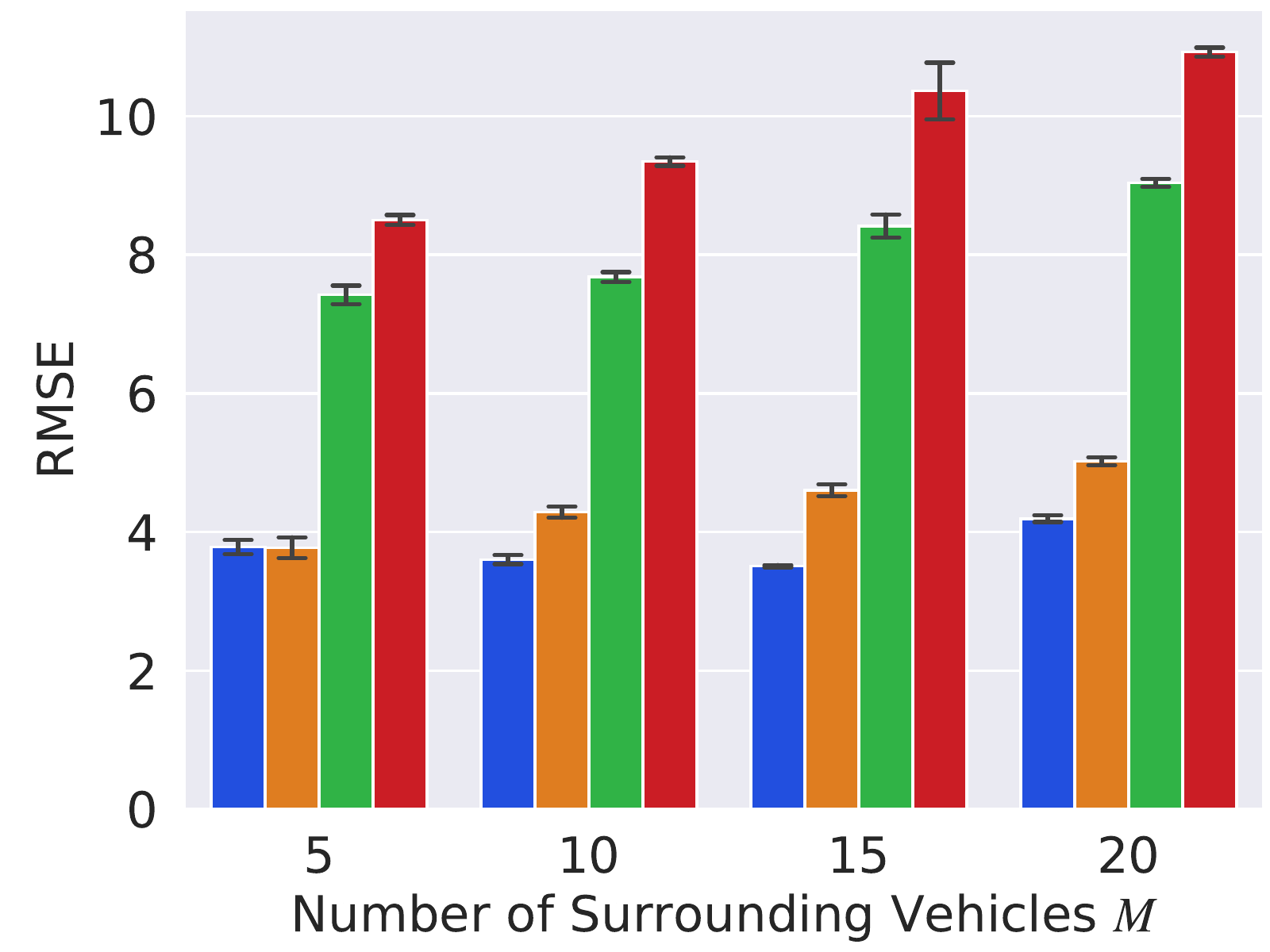}}
        \subfloat[]{\label{fig04:task2}\includegraphics[width=0.30\textwidth]{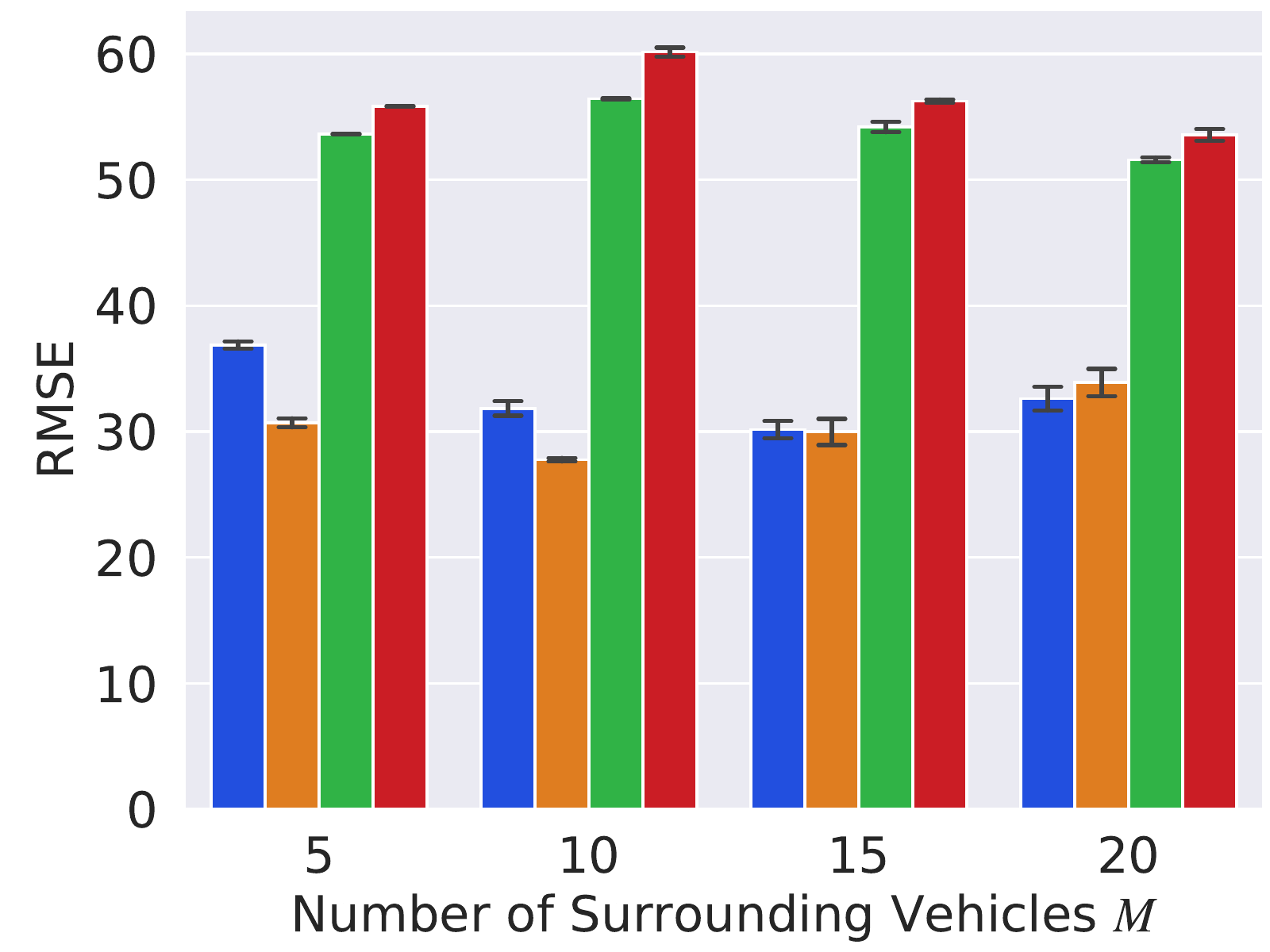}} 
        \subfloat[]{\label{fig04:task3}\includegraphics[width=0.30\textwidth]{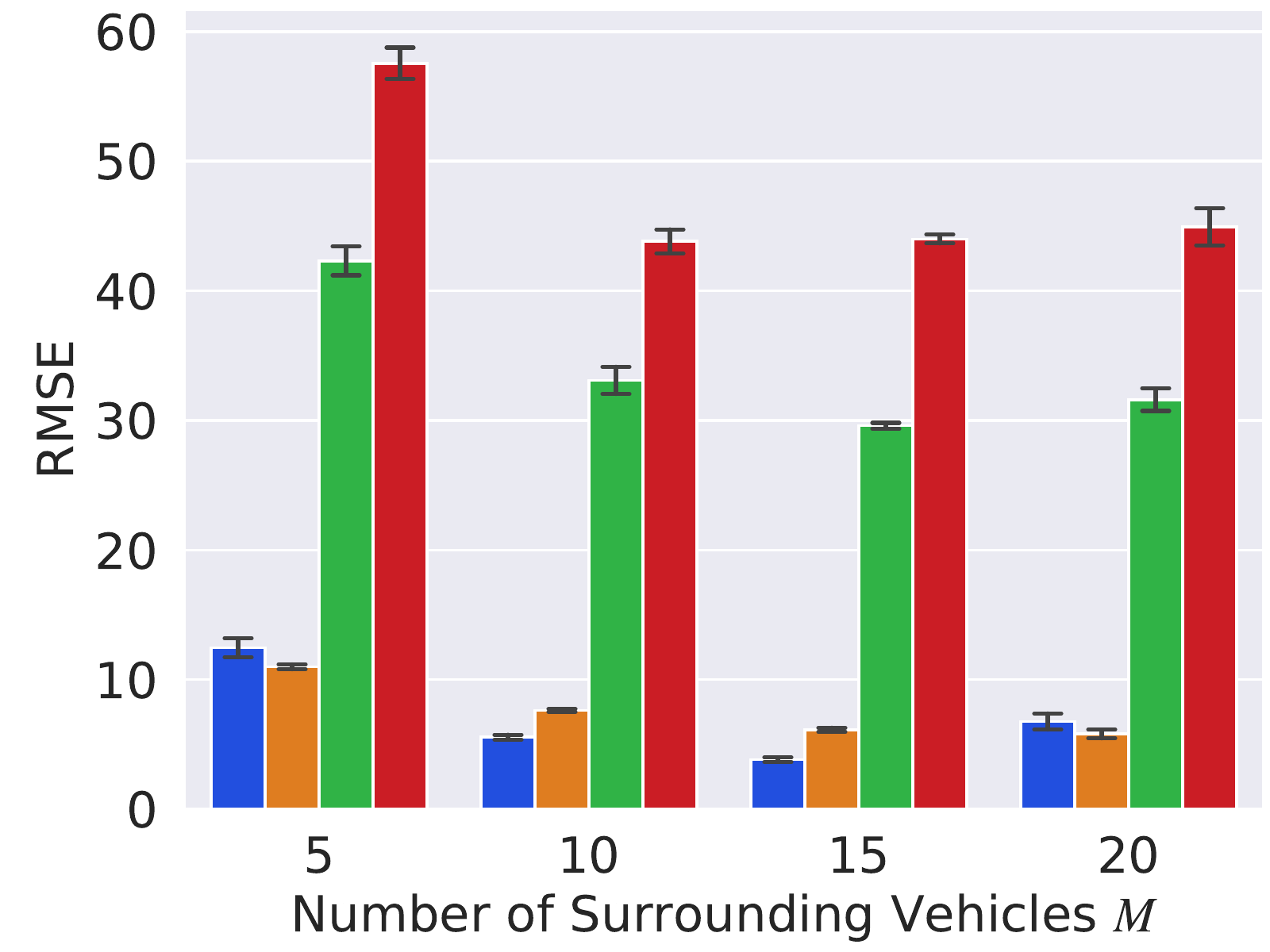}}\\
        \subfloat[]{\label{fig04:task4}\includegraphics[width=0.30\textwidth]{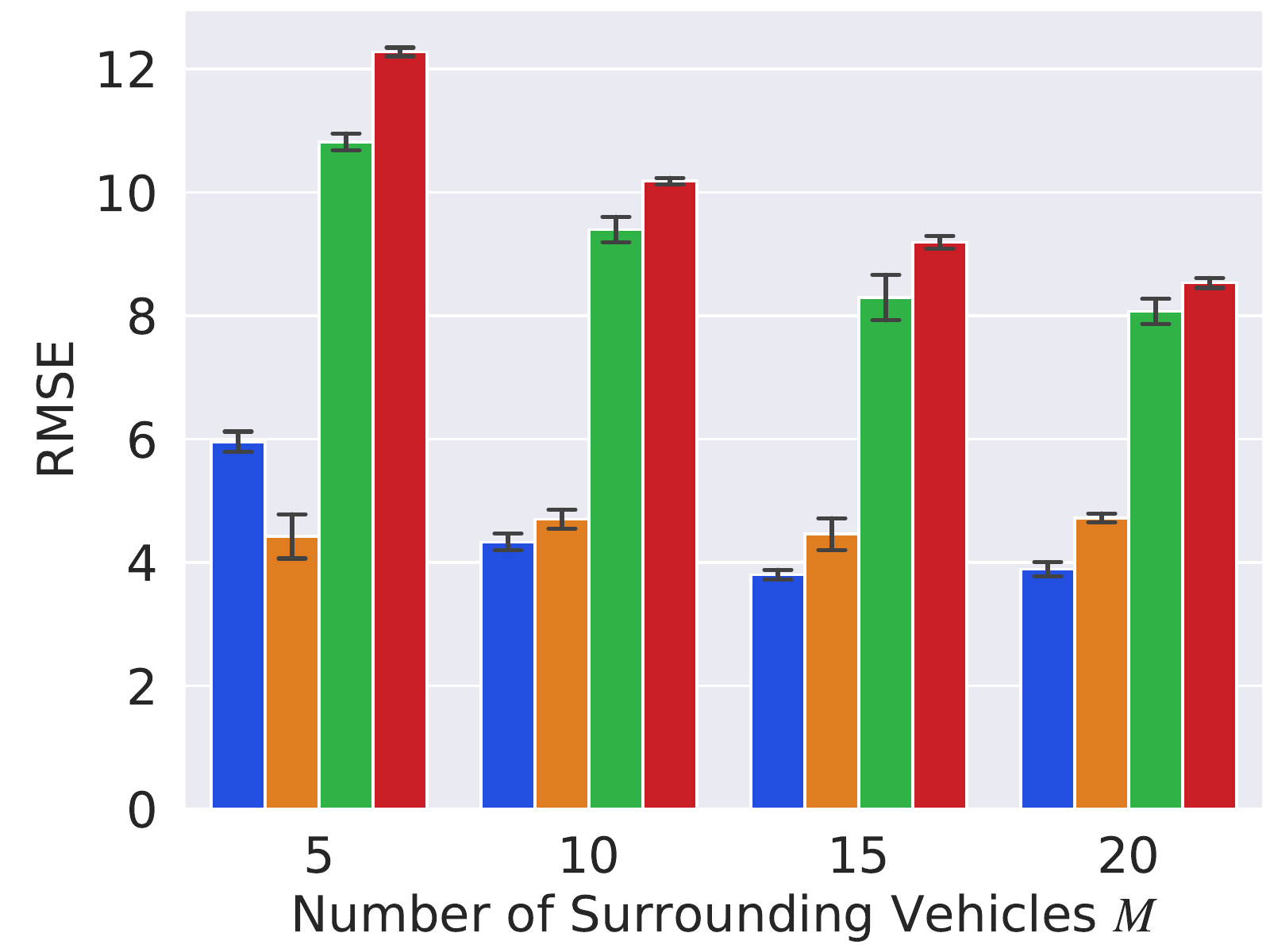}} 
        \subfloat[]{\label{fig04:task5}\includegraphics[width=0.30\textwidth]{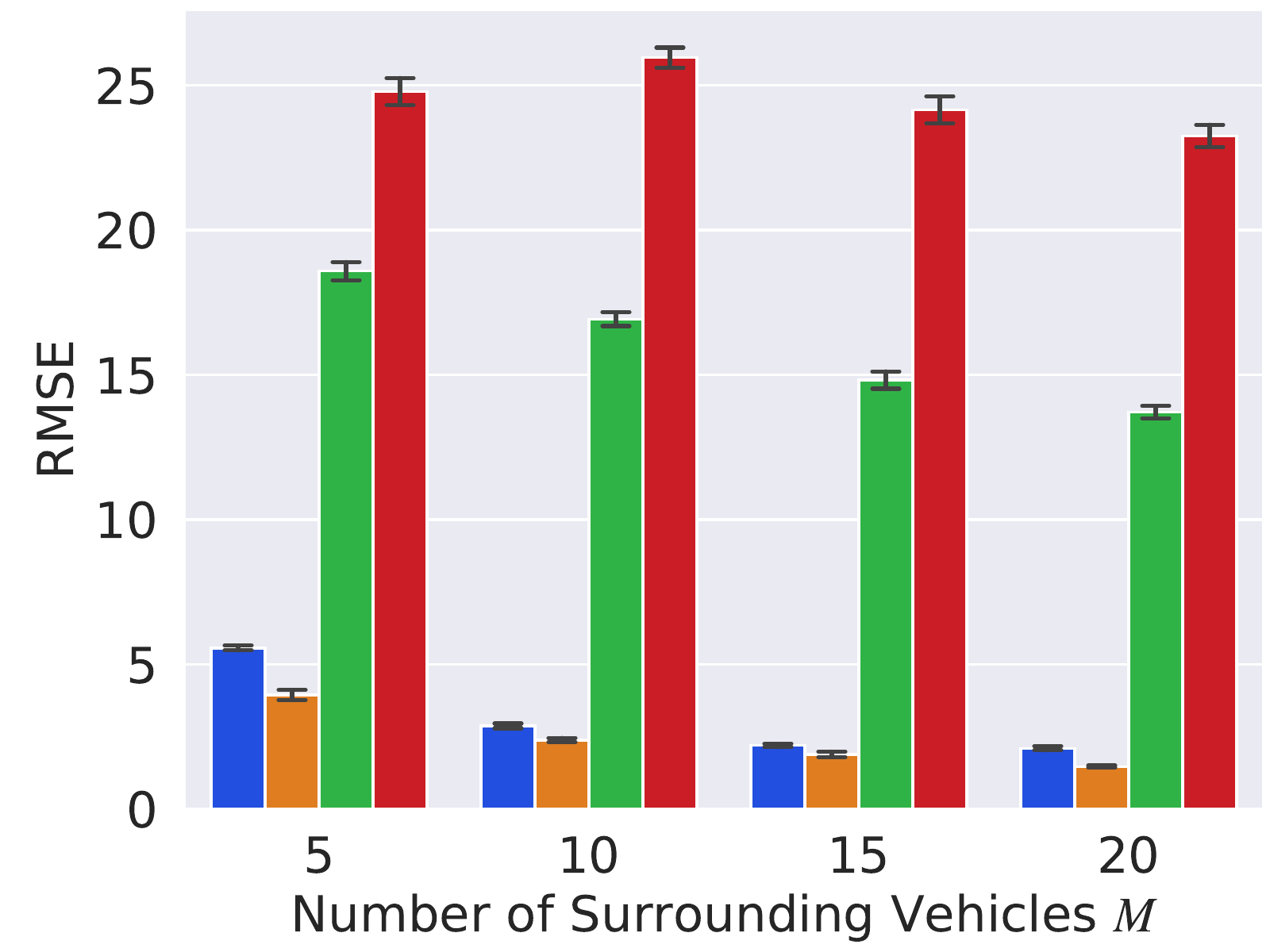}}
        \subfloat[]{\label{fig04:task6}\includegraphics[width=0.30\textwidth]{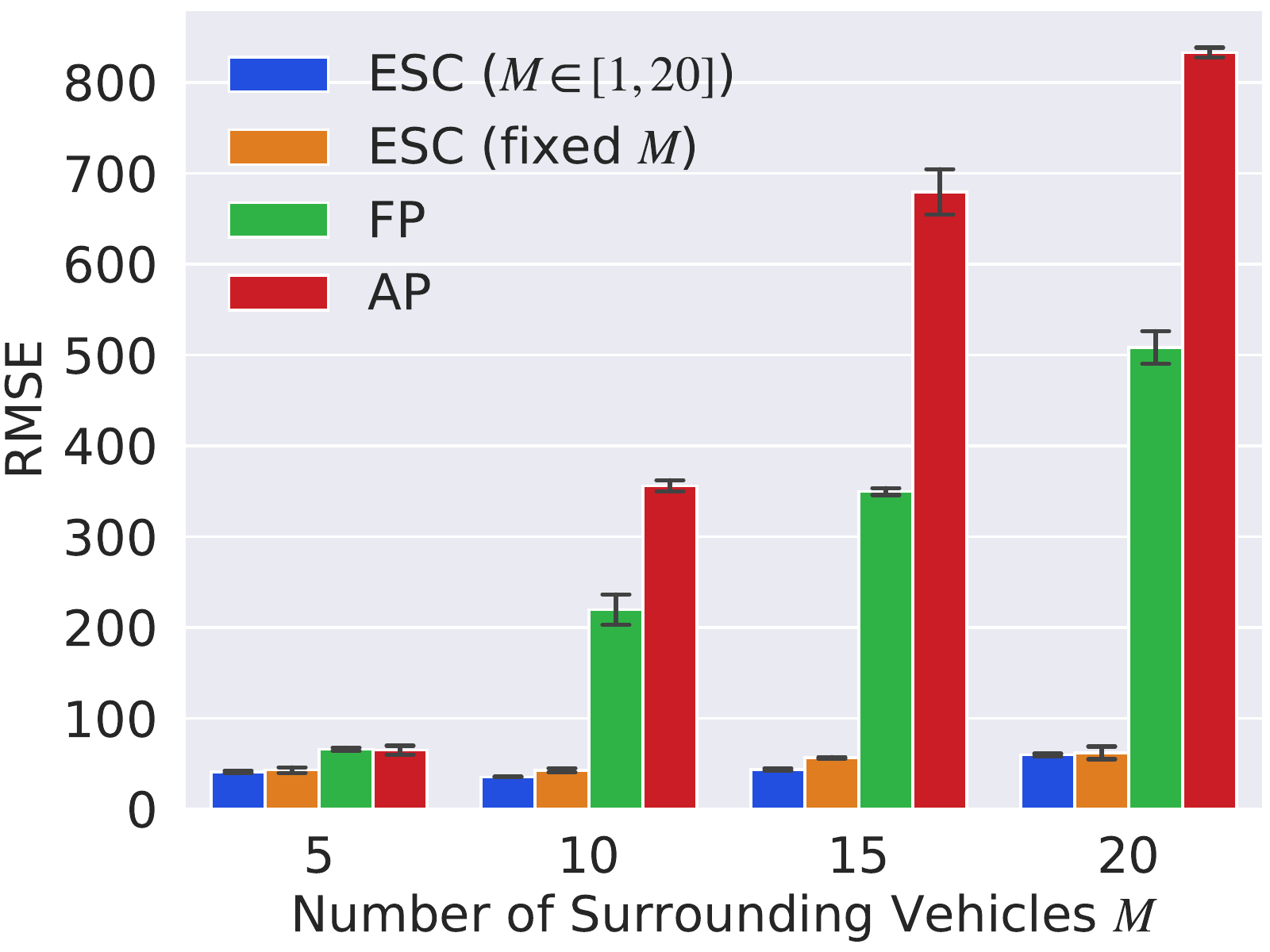}} \\
        \caption{Average final RMSE of different representation methods on all benchmarks. (a) Benchmark 1. (b) Benchmark 2. (c) Benchmark 3. (d) Benchmark 4. (e) Benchmark 5. (f) Benchmark 6.  }
    \label{fig04:loss_comparison}
\end{figure*}

Fig. \ref{fig04:loss_comparison} and Table \ref{tab04:loss_comparison} display the final RMSE under each experimental setting. Results show that the proposed ESC method outperforms or matches two baselines in all benchmarks and cases. Among all the cases, the RMSE of the FP method is 20.7\% lower than that of the AP method on average. This is because the predetermined order $o$ helps to reduce the sample complexity. Compared with the AP and FP methods, ESC with fixed $M$ achieves an average error reduction of 62.2\% and 67.5\%, respectively. On the one hand, it is obvious that when considering the same number of surrounding vehicles ($M=5$, $10$, $15$, or $20$), the performance of ESC (fixed $M$) is much better than AP and FP. This indicates that ESC is more suitable to represent the surrounding vehicles set $\SPACE{X}$ due to its permutation-invariance property and continuity. Compared with FP, ESC also eliminates the requirement of manually designed sorting rules. On the other hand, the learning accuracy of ESC with variable-size sets (ESC ($M\in [1,20]$)) is comparable to that with fixed set size (ESC (fixed $M$)). Therefore, it suggests that the ESC method is capable of representing variable-size sets, thereby eliminating the burden of training different approximation NNs for scenarios with different numbers of surrounding vehicles. To conclude, experimental results indicate that the proposed ESC method improves the representation ability of driving observation.

\subsection{Future Work}
In this paper, the policy NN and feature NN are updated to approximate the designed target policy in Table \ref{tab04:task} under the supervised learning framework. The proposed ESC approach has potential to be adopted in actual state representation applications of autonomous driving based on supervised learning or RL. To this end, we first need to normalize all indicators of surrounding vehicles to the same range. Then, for supervised learning based decision-making, we can learn both policy and feature NNs using the labeled data $\{\mathcal{O}_i, a_i^*\}$, where $a^*$ is the labeled action. For RL-based decision-making, we can iteratively find nearly optimal policy and feature NNs using the samples $\{\mathcal{O}_t, a_t, r_t, \mathcal{O}_{t+1}\}$ collected from the interaction between the ego vehicle and environments, where $r$ represents the reward function. In the future, we will focus on the applications of ESC in supervised learning based or RL-based autonomous driving.

\section{Conclusions}
\label{sec:conclusion}
In this paper, we first analyze the dimension sensitivity and permutation sensitivity issues faced by existing AP and FP representation methods. Due to dimension sensitivity, T2E methods based on  AP or FP representation are only  valid when the number of surrounding vehicles is fixed. Due to the permutation sensitivity, AP and FP methods suffer from high sample complexity and policy discontinuity respectively. Both dimension sensitivity and permutation sensitivity will damage the policy learning accuracy and limit the applicability of T2E decision-making in different driving scenarios. 

To overcome this problem, we propose the ESC state representation method to describe the environment observation for decision-making in autonomous driving. The proposed ESC method employs a feature NN to encode the real-valued feature of each surrounding vehicle into an encoding vector, and then adds these vectors to obtain the representation vector of the set of surrounding vehicles. By concatenating the set representation with other variables, such as indicators of the ego vehicle and road, we achieve a fixed-dimensional and permutation-invariance state representation. We have further proved that there exists an over-parameterized feature NN such that the ESC state representation is injective if the output dimension of the feature NN is greater than the number of variables of all surrounding vehicles. Besides, by taking the ESC representation as policy inputs, the nearly optimal feature NN and policy NN can be found by simultaneously optimizing them using gradient-based updating. Experiments demonstrate the proposed ESC method improves the  representation ability of driving observation, leading to a reduction of 62.2\% in policy learning error compared with the fixed-permutation representation method.

\section*{Acknowledgment}
The authors are grateful to the Editor-in-Chief, the Associate Editor, and anonymous reviewers for their valuable comments.


\ifCLASSOPTIONcaptionsoff
  \newpage
\fi
\bibliographystyle{ieeetr}
\bibliography{ref}



%


\begin{IEEEbiography}[{\includegraphics[width=1in,height=1.25in,clip,keepaspectratio]{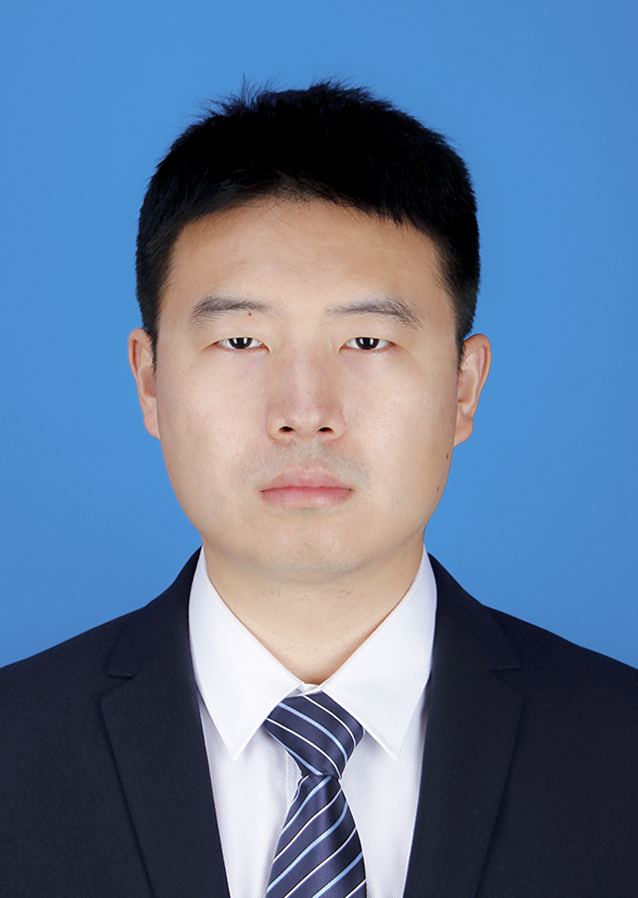}}]{Jingliang Duan}
received the B.S. degree from the College of Automotive Engineering, Jilin University, Changchun, China, in 2015. He studied as a visiting student researcher in Department of Mechanical Engineering, University of California, Berkeley, USA, in 2019. He received his Ph.D. degree in the School of Vehicle and Mobility, Tsinghua University, Beijing, China, in 2021. He is currently a research fellow in the Department of Electrical and Computer Engineering, National University of Singapore, Singapore. His research interests include decision and control of autonomous vehicle, reinforcement learning and adaptive dynamic programming, and driver behaviour analysis.
\end{IEEEbiography}

\vskip -2\baselineskip plus -1fil

\begin{IEEEbiography}[{\includegraphics[width=1in,height=1.25in,clip,keepaspectratio]{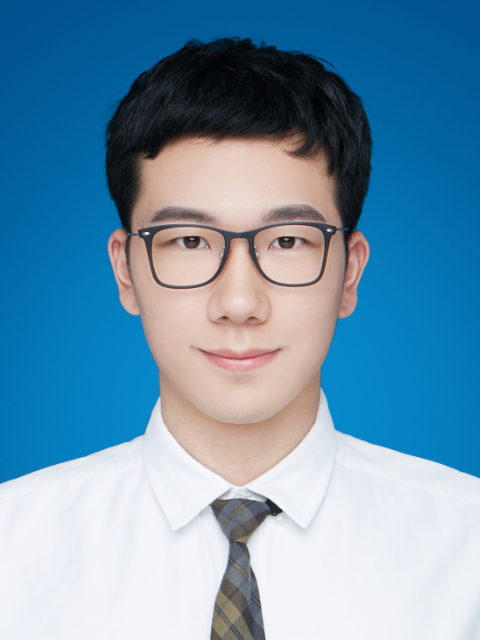}}]{Dongjie Yu}
received the B.S. degree from the School of Vehicle and Mobility, Tsinghua University, Beijing, China, in 2020. He is currently pursuing the Ph.D. degree with the School of Vehicle and Mobility, Tsinghua University, Beijing. 
His research interests include decision-making of autonomous vehicles and reinforcement learning.
\end{IEEEbiography}

\vskip -2\baselineskip plus -1fil

\begin{IEEEbiography}[{\includegraphics[width=1in,height=1.25in,clip,keepaspectratio]{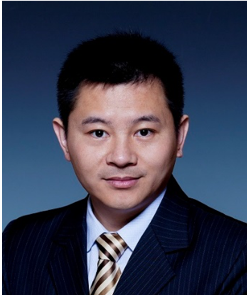}}]{Shengbo Eben Li}
(SM'16) received the M.S. and Ph.D. degrees from Tsinghua University in 2006 and 2009. He worked at Stanford University, University of Michigan, and University of California, Berkeley. He is currently a tenured professor at Tsinghua University. His active research interests include intelligent vehicles and driver assistance, reinforcement learning and distributed control, optimal control and estimation, etc.

He is the author of over 100 journal/conference papers, and the co-inventor of over 20 Chinese patents. He was the recipient of Best Paper Award in 2014 IEEE ITS Symposium, Best Paper Award in 14th ITS Asia Pacific Forum, National Award for Technological Invention in China (2013), Excellent Young Scholar of NSF China (2016), Young Professorship of Changjiang Scholar Program (2016). He is now the IEEE senior member and serves as associated editor of IEEE ITSM and IEEE Trans. ITS, etc.
\end{IEEEbiography}

\vskip -2\baselineskip plus -1fil

\begin{IEEEbiography}[{\includegraphics[width=1in,height=1.25in,clip,keepaspectratio]{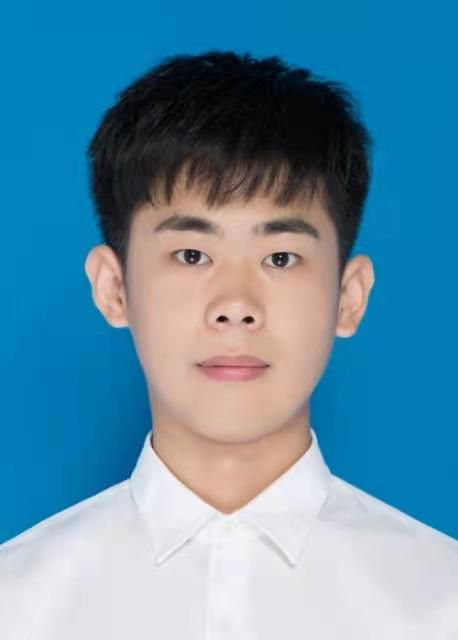}}]{Wenxuan Wang}
received his B.S. degree in vehicle engineering from Beijing  Institute Of Technology, Beijing, China, in 2018. He is currently a member in the State Key Laboratory of Automotive Safety and Energy, School of Vehicle and Mobility, Tsinghua University. His current research interests include decision-making and control of automated vehicles, and reinforcement learning algorithms.
\end{IEEEbiography}

\vskip -2\baselineskip plus -1fil

\begin{IEEEbiography}[{\includegraphics[width=1in,height=1.25in,clip,keepaspectratio]{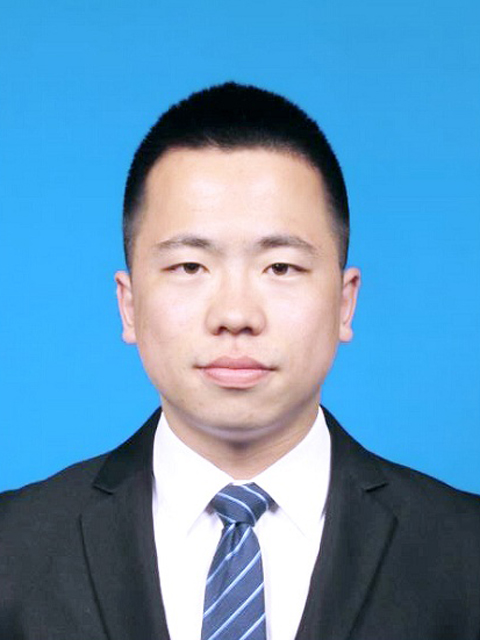}}]{Yangang Ren}
received the B.S. degree from the Department of Automotive Engineering, Tsinghua University, Beijing, China, in 2018. He is currently pursuing his Ph.D. degree in the School of Vehicle and Mobility, Tsinghua University, Beijing, China. His research interests include decision and control of autonomous driving, reinforcement learning, and adversarial learning.
\end{IEEEbiography}

\vskip -2\baselineskip plus -1fil

\begin{IEEEbiography}[{\includegraphics[width=1in,height=1.25in,clip,keepaspectratio]{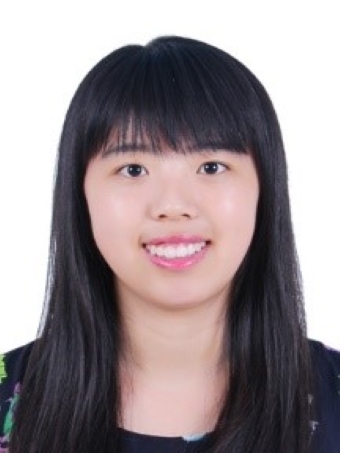}}]{Ziyu Lin}
received the B.S. degree in automotive engineering from China Agricultural University, Beijing, China, in 2017. She is currently pursuing the Ph.D. degree with the School of Vehicle and Mobility, Tsinghua University, Beijing. Her current reseach interests include model-based reinforcement learning, approximate dynamic programming, and model predictive control and distributed control.  She was a recipient of the Best Paper Award at the IEEE 2020 3rd International Conference on Unmanned Systems (ICUS), and Best Presentation Award on IEEE 2021 International Conference on  Computer Control and Robotics.
\end{IEEEbiography}

\vskip -2\baselineskip plus -1fil
\begin{IEEEbiography}[{\includegraphics[width=1in,height=1.25in,clip,keepaspectratio]{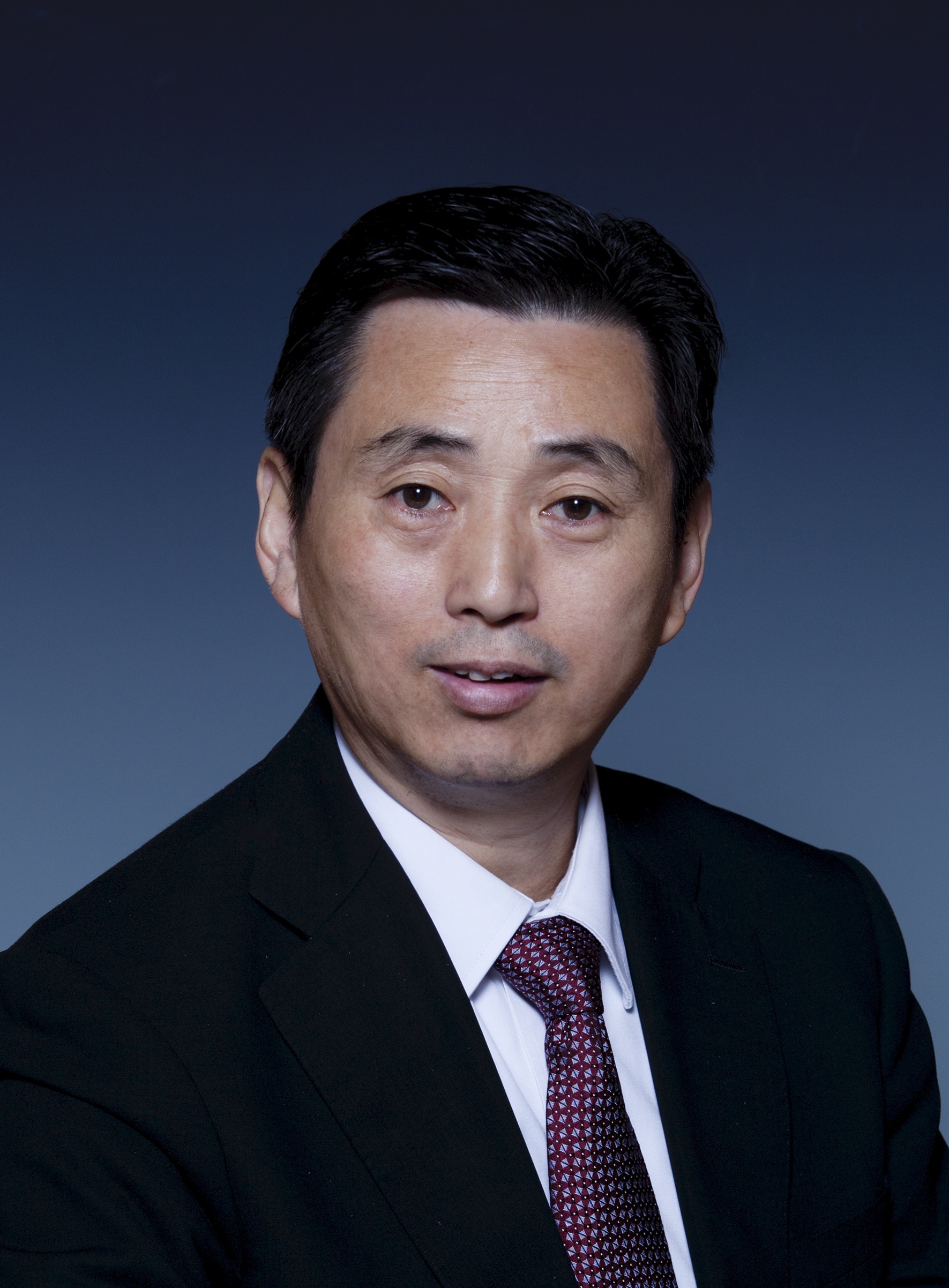}}]{Bo Cheng}
received the B.S. and M.S. degrees in automotive engineering from Tsinghua University, Beijing, China, in 1985 and 1988, respectively, and the Ph.D. degree in mechanical engineering from the University of Tokyo, Tokyo, Japan, in 1998. He is currently a Professor with School of Vehicle and Mobility, Tsinghua University, and the Dean of Tsinghua University–Suzhou Automotive Research Institute. He is the author of more than 100 peer-reviewed journal/conference papers and the co-inventor of 40 patents. His active research interests include autonomous vehicles, driver-assistance systems, active safety, and vehicular ergonomics, among others. 
\end{IEEEbiography}



%






\end{document}